\documentclass{article}

    \usepackage[final]{neurips_2022}

\usepackage[utf8]{inputenc} 
\usepackage[T1]{fontenc}    
\usepackage{url}            
\usepackage{booktabs}       
\usepackage{amsfonts}       
\usepackage{nicefrac}       
\usepackage{microtype}      
\usepackage{xcolor}         

\title{Robust Reinforcement Learning using Offline Data}

\usepackage{amsmath,amsthm,amstext,amsfonts,amssymb,mathrsfs}
\usepackage{algorithm, algorithmic}
\usepackage{enumitem}
\usepackage{graphicx}
\usepackage{textcomp}
\usepackage{xcolor}
\usepackage{bbm}
\usepackage{caption}
\usepackage{picture,xcolor}
\usepackage{url}
\usepackage[colorlinks=true, urlcolor= black, linkcolor=blue,citecolor=black,pagebackref=true]{hyperref} 

\usepackage{ulem}
\usepackage{wrapfig}
\usepackage{lipsum}

\allowdisplaybreaks

\newtheorem{lemma}{Lemma}

\newtheorem{assumption}{Assumption}

\newtheorem{proposition}{Proposition}
\newtheorem{theorem}{Theorem}
\newtheorem{definition}{Definition}

\theoremstyle{remark} 
\newtheorem{remark}{Remark}

\newcommand{\dd}{\mathrm{d}\,}
\newcommand{\E}{\mathbb{E}}
\newcommand{\pr}{\mathbb{P}}
\newcommand{\R}{\mathbb{R}}
\newcommand{\cS}{\mathcal{S}} 
\newcommand{\cA}{\mathcal{A}}
\newcommand{\ScA}{\cS\times\cA}
\newcommand{\cP}{\mathcal{P}}

\newcommand{\cN}{\mathcal{N}}
\newcommand{\cF}{\mathcal{F}}
\newcommand{\cG}{\mathcal{G}}
\newcommand{\cD}{\mathcal{D}}
\newcommand{\cO}{\mathcal{O}}
\newcommand{\cX}{\mathcal{X}}
\newcommand{\cM}{\mathcal{M}}
\newcommand{\cH}{\mathcal{H}}
\newcommand{\cC}{\mathcal{C}}

\renewcommand{\epsilon}{\varepsilon}
\renewcommand{\phi}{\varphi}

\newcommand{\norm}[1]{\left\| #1 \right\|}%

\DeclareMathOperator*{\argmax}{arg\,max}
\DeclareMathOperator*{\argmin}{arg\,min}

\newcommand{\bluetext}[1]{{\leavevmode\color{blue}#1}}

\usepackage{cleveref}

\author{Kishan Panaganti$^{1}$, Zaiyan Xu$^{1}$, Dileep Kalathil$^{1}$, Mohammad Ghavamzadeh$^{2}$ \\
  $^{1}$Texas A\&M University, $^{2}$Google Research. \\Emails:  \texttt{\{kpb, zxu43 ,dileep.kalathil\}@tamu.edu}, \texttt{ghavamza@google.com}
}

\begin{document}

\maketitle

\begin{abstract}
  The  goal of robust reinforcement learning (RL)  is to learn a policy that is robust against the uncertainty in model parameters. Parameter uncertainty commonly occurs in many real-world RL applications due to  simulator modeling errors,  changes in the real-world system dynamics over time, and  adversarial disturbances. Robust RL is typically formulated as a max-min problem, where the objective is to learn the policy that maximizes the value  against the worst possible models that lie in an uncertainty set. In this work, we propose a  robust RL algorithm called Robust Fitted Q-Iteration (RFQI), which uses only an offline dataset to learn the optimal robust policy.  Robust RL with offline data is significantly more challenging than its non-robust counterpart because of the minimization over all models present in the robust Bellman operator. This poses challenges in offline data collection,  optimization over the models, and unbiased estimation. In this work, we propose a systematic approach to overcome these challenges, resulting in our RFQI algorithm. We prove that RFQI learns a near-optimal robust policy under standard assumptions and demonstrate its superior performance on standard benchmark problems. 
\end{abstract}

\section{Introduction}
\label{sec:introduction}

Reinforcement learning (RL) algorithms often require a large number of data samples to learn a control policy. As a result, training them directly on the real-world systems is expensive and potentially dangerous. This problem is typically overcome by training them on a simulator (online RL) or using a pre-collected offline dataset (offline RL). The offline dataset is usually collected either from a sophisticated simulator of the real-world system or from the historical measurements. The trained RL policy is then deployed assuming that the training environment, the simulator or the offline data, faithfully represents the model of the real-world system. This assumption is often incorrect due to multiple factors such as the approximation errors incurred while modeling, changes in the real-world parameters over time and possible adversarial disturbances in the real-world. For example, the standard simulator settings of the sensor noise, action delay, friction, and mass of a mobile robot can be different from that of the actual real-world robot, in addition to changes in the terrain, weather conditions, lighting, and obstacle densities of the testing environment. 
Unfortunately, the current RL control policies can fail dramatically when faced with even mild changes in the training and testing environments \citep{sunderhauf2018limits, tobin2017domain, peng2018sim}.

The goal in robust RL is to learn a policy that is robust against the model parameter mismatches between the training and testing environments. The robust planning problem is formalized using the framework of Robust Markov Decision Process (RMDP) \citep{iyengar2005robust,nilim2005robust}. Unlike the standard MDP which considers a single model (transition probability function), the RMDP formulation considers a set of models which is called the \textit{uncertainty set}. The goal is to find an optimal robust policy that performs the best under the worst possible model in this uncertainty set.  
The minimization over the uncertainty set makes the robust MDP and robust RL problems significantly more challenging than their non-robust counterparts.

In this work, we study the problem of developing a robust RL algorithm with provably optimal performance for an RMDP with arbitrarily large state spaces, using only offline data with function approximation. Before stating the contributions of our work, we provide a brief overview of the results in offline and robust RL that are directly related to ours. We leave a more thorough discussion on related works to Appendix~\ref{sec:app:related-works}.

\looseness=-1\textbf{Offline RL:} Offline RL  considers the problem of learning the optimal policy only using a pre-collected (offline) dataset. Offline RL problem has been addressed extensively in the literature \citep{antos2008learning, bertsekas2011approximate,  lange2012batch, chen2019information, xie2020q, levine2020offline, xie2021bellman}. Many recent works develop deep RL algorithms and heuristics for the offline RL problem, focusing on the algorithmic and empirical aspects \citep{fujimoto2019off, kumar2019stabilizing, kumar2020conservative, yu2020mopo, zhang2021towards}. A number of theoretical work focus on analyzing the variations of Fitted Q-Iteration (FQI) algorithm \citep{gordon1995stable, ernst2005tree}, by   identifying the necessary and sufficient conditions for the learned policy to be approximately optimal and characterizing the performance in terms of sample complexity \citep{munos08a, farahmand2010error, lazaric2012finite, chen2019information, liu2020provably, xie2021bellman}. All these works assume that the offline data is generated according to a single model and the goal is to find the optimal policy for the MDP with the same model. In particular, none of these works consider the \textit{offline robust RL problem} where the offline data is generated according to a (training) model which can be different from the one in testing, and the goal is to learn a policy that is robust w.r.t.~an uncertainty set.

\textbf{Robust RL:} The RMDP framework was first introduced in~\citet{iyengar2005robust, nilim2005robust}. The RMDP problem has been analyzed extensively in the literature \citep{xu2010distributionally, wiesemann2013robust, yu2015distributionally,   mannor2016robust, russel2019beyond} providing computationally efficient algorithms,   but these works are limited to the {planning} problem.  Robust RL algorithms with provable guarantees have also been proposed \citep{lim2013reinforcement,tamar2014scaling,roy2017reinforcement,panaganti2020robust, wang2021online}, but they are limited to tabular or linear function approximation settings and only provide asymptotic convergence guarantees. Robust RL problem has also been addressed using deep RL methods \citep{pinto2017robust,  derman2018soft, derman2020bayesian, Mankowitz2020Robust, zhang2020robust}. However, these works  do not provide any theoretical guarantees on the performance of the learned policies.

The works that are closest to ours are by \citet{zhou2021finite,yang2021towards,panaganti22a} that address the robust RL problem in a tabular setting under the generative model assumption. Due to the generative model assumption, the offline data has the same uniform number of samples corresponding to each and every state-action pair, and tabular setting allows the estimation of the uncertainty set followed by solving the planning problem. Our work is significantly different from these in the following way: $(i)$ we consider a robust RL problem with arbitrary large state space, instead of the small tabular setting, $(ii)$ we consider a true offline RL setting where the state-action pairs are sampled according to an arbitrary distribution, instead of using the generative model assumption, $(iii)$ we focus on a function approximation approach where the goal is to directly learn optimal robust value/policy using function approximation techniques, instead of solving the  tabular planning problem with the estimated model. \textit{To the best of our knowledge, this is the first work that addresses the offline robust RL problem with arbitrary large state space  using function approximation, with provable guarantees on the performance of the learned policy.}

\textbf{Offline Robust RL: Challenges and Our Contributions:} Offline robust RL is significantly more challenging than its non-robust counterpart mainly because of the following key difficulties. \\
$(i)$ Data generation: The optimal robust policy is computed by taking the infimum over all models in the uncertainty set $\cP$. However, generating data according to all models in $\cP$ is clearly infeasible. It may only be possible to get the data from a nominal (training) model $P^{o}$. \textit{How do we use the data  from a  nominal model to account for the behavior of all the models in the uncertainty set $\cP$?} \\
$(ii)$ Optimization over the uncertainty set $\cP$: The robust Bellman operator (defined in \eqref{eq:robust-bellman-primal}) involves a minimization over $\cP$, which is a significant computational challenge. Moreover, the uncertainty set $\cP$ itself is unknown in the RL setting. \textit{How do we solve the optimization over $\cP$?} \\
$(iii)$  Function approximation: Approximation of the robust Bellman update requires a modified target function which also depends on the approximate solution of the optimization over the uncertainty set. \textit{How do we perform the offline RL update accounting for both approximations? }

\looseness=-1 As the \textit{key technical contributions} of this work, we first derive a dual reformulation of the robust Bellman operator which replaces the expectation w.r.t.~all models in the uncertainty set $\cP$ with an expectation  only w.r.t.  the nominal (training) model $P^o$. This enables using the offline data generated by $P^o$ for learning, without relying on high variance importance sampling techniques to account for all models in $\cP$. Following the same reformulation, we then show that the  optimization problem over $\cP$ can be further reformulated as functional optimization. We solve this functional optimization problem using empirical risk minimization and obtain  performance guarantees using the Rademacher complexity based bounds. We then use the approximate solution obtained from the empirical risk minimization to generate modified target samples that are then used to approximate robust Bellman update through a generalized least squares approach with provably bounded errors. Performing these operations iteratively results in our proposed Robust Fitted Q-Iteration (RFQI) algorithm, for which we prove that its learned policy achieves non-asymptotic and approximately optimal performance guarantees.

\textbf{Notations:} For a set $\cX$, we denote its cardinality as $|\cX|$. The set of probability distribution over $\cX$  is denoted  as $\Delta(\cX)$, and its power set sigma algebra as $\Sigma(\cX)$. For any $x \in \mathbb{R}$, we denote $\max\{x, 0\}$ as $(x)_{+}$.
For any function $f:\ScA\rightarrow \R$, state-action distribution $\nu\in\Delta({\ScA})$, and real number $p\geq 1$, the $\nu$-weighted $p$-norm of $f$ is defined as $\|f\|_{p,\nu} = \E_{s,a\sim\nu}[|f(s,a)|^p]^{1/p}$.

\section{Preliminaries}
\label{sec:preliminaries}

A Markov Decision Process (MDP) is a tuple $(\cS, \cA, r, P, \gamma, d_0)$, where $\cS$ is the  state space, $\cA$ is the  action space,  $r: \cS\times \cA\rightarrow \R$ is the reward function, $\gamma \in (0, 1)$ is the discount factor, and $d_0\in \Delta({\cS})$ is the initial state distribution. The transition probability function $P_{s,a}(s')$ is the probability of transitioning to state $s'$ when  action $a$ is taken at state $s$. In the literature, $P$ is also called the \textit{model} of the MDP. We consider a setting where  $|\cS|$ and $|\cA|$ are finite but can be arbitrarily large. We will also assume that $r(s,a) \in [0,1]$, for all $(s,a) \in \ScA$,  without loss of generality.  A policy $\pi : \cS \rightarrow \Delta({\cA})$  is a conditional distribution over actions given a state. The value function $V_{\pi, P}$ and the state-action value function $Q_{\pi,P}$ of a policy $\pi$ for an MDP with model $P$ are defined as
\vspace{-0.1cm}
\begin{align*}
V_{\pi, P}(s) = \E_{\pi, P}[\sum^{\infty}_{t=0} \gamma^{t} r(s_{t}, a_{t}) ~|~ s_{0} = s], \quad  Q_{\pi, P}(s, a) = \E_{\pi, P}[\sum^{\infty}_{t=0} \gamma^{t} r(s_{t}, a_{t}) ~|~ s_{0} = s, a_{0} = a], 
\end{align*}
where the expectation is over the randomness induced by the policy $\pi$ and model $P$. 
The optimal value function $V^{*}_{P}$ and the optimal policy  $\pi^{*}_{P}$ of an MDP with the model $P$ are defined as $V^{*}_{P} = \max_{\pi} V_{\pi, P}$ and $\pi^{*}_{P} = \argmax_{\pi} V_{\pi, P}$. 
The optimal state-action value function is given by $Q^{*}_{P} = \max_{\pi} Q_{\pi, P}$. The optimal policy can  be obtained as  $\pi^{*}_{P}(s) = \argmax_{a} Q_P^{*}(s, a)$. The discounted state-action occupancy of a policy $\pi$ for an MDP with model $P$, denoted as $d_{\pi, P} \in \Delta({\ScA})$, is defined as $d_{\pi,P}(s,a)=(1-\gamma) \E_{\pi,P}[\sum_{t=0}^\infty \gamma^{t} \mathbbm{1}(s_t=s,a_t=a)]$.

\textbf{Robust Markov Decision Process (RMDP):} Unlike the standard MDP which considers a single model (transition probability function), the RMDP formulation considers a set of models. We refer to this set as the \textit{uncertainty set} and denote it as $\cP$. We consider  $\cP$ that satisfies the standard \textit{$(s,a)$-rectangularity condition} \citep{iyengar2005robust}. We note that a similar {uncertainty set} can be considered for the reward function at the expense of additional notations. However, since the analysis will be similar and the sample complexity guarantee will be identical up to a constant factor, without loss of generality,  we assume that the reward function  is known and deterministic.  

We specify an RMDP as $M = (\cS, \cA, r, \cP, \gamma, d_0)$, where the uncertainty set $\mathcal{P}$ is typically defined as
\begin{align}
    \label{eq:uncertainty-set}
    \mathcal{P} = \otimes_{(s,a) \in \ScA }\, \mathcal{P}_{s,a}, \quad \text{where} \;\;
    \mathcal{P}_{s,a} =  \{ P_{s,a} \in \Delta(\cS)~:~ D(P_{s,a}, P^o_{s,a}) \leq \rho  \},
\end{align}
$P^{o} = (P^o_{s,a}, (s, a) \in \ScA)$ is the \textit{nominal model}, $D(\cdot, \cdot)$ is a distance metric between two probability distributions, and $\rho>0$ is the radius of the uncertainty set that indicates the level of robustness. The nominal model $P^{o}$ can be thought as the model of  the training environment. It is either the model of the simulator on which the (online) RL algorithm is trained, or in our setting, it is the model according to which the offline data is generated. The uncertainty set  $\cP$ \eqref{eq:uncertainty-set} is the set of all valid transition probability functions (valid testing models) in the neighborhood of the nominal model $P^{o}$, which by definition satisfies $(s,a)$-rectangularity condition \citep{iyengar2005robust}, where the neighborhood is defined using the distance metric $D(\cdot, \cdot)$ and radius $\rho$. In this work, we consider the \textit{Total Variation (TV) uncertainty set} defined using the TV distance, i.e.,  $D(P_{s,a}, P^{o}_{s,a}) = (1/2) \|P_{s,a} - P^{o}_{s,a} \|_{1}$.

The RMDP problem is to find the optimal robust policy which maximizes the value against the worst possible model in the uncertainty set $\mathcal P$. The \textit{robust value function} $V^{\pi}$ corresponding to a policy $\pi$ and the \textit{optimal robust value function} $V^{*}$ are defined as \citep{iyengar2005robust,nilim2005robust}
\begin{align}
\label{eq:robust-value-function}
V^{\pi} = \inf_{P \in \mathcal{P}} ~V_{\pi, P},\qquad V^{*} = \sup_{\pi} \inf_{P \in \mathcal{P}} ~V_{\pi, P} . 
\end{align} 
The \textit{optimal robust policy} $\pi^{*}$ is such that the robust value function corresponding to it matches the optimal robust value function, i.e., $V^{\pi^*} =V^{*} $. It is known that there exists a deterministic optimal policy \citep{iyengar2005robust} for the RMDP. The \textit{robust Bellman operator} is defined  as \citep{iyengar2005robust}
\begin{align}
\label{eq:robust-bellman-primal}
    (T Q)(s, a) = r(s, a) + \gamma \inf_{P_{s,a} \in \mathcal{P}_{s,a}} \mathbb{E}_{s' \sim P_{s,a}} [ \max_{b} Q(s', b)]. 
\end{align}
It is known that $T$ is a contraction mapping in the infinity norm and hence it has a unique fixed point $Q^{*}$ with $V^{*}(s) = \max_{a} Q^{*}(s,a)$ and $\pi^{*}(s) = \argmax_{a} Q^{*}(s,a)$ \citep{iyengar2005robust}. The \textit{Robust Q-Iteration (RQI)}  can now be defined using the robust Bellman operator as  $Q_{k+1} = T Q_{k}$. Since $T$ is a contraction, it follows that $Q_{k} \rightarrow Q^{*}$. So, RQI can be used to compute (solving the planning problem) $Q^{*}$ and $\pi^{*}$ in the tabular setting with a known $\cP$. Due to the  optimization over the uncertainty set $\cP_{s,a}$ for each $(s,a)$ pair, solving the planning problem in RMDP using RQI is much more computationally intensive than solving it in MDP using Q-Iteration.

\textbf{Offline RL:}  Offline RL considers the problem of learning the optimal policy of an MDP when the  algorithm does not have direct access to the environment  and  cannot generate  data samples in an online manner. For learning the optimal policy $\pi^{*}_{P}$ of an MDP with model $P$, the  algorithm will only have  access to an offline   dataset $\cD_{P}=\{(s_i,a_i,r_{i}, s'_{i})\}_{i=1}^N$, where $(s_{i}, a_{i}) \sim \mu$, $\mu \in \Delta (\ScA)$ is some  distribution, and $s'_{i} \sim P_{s_{i},a_{i}}$. \textit{Fitted Q-Iteration (FQI)} is a popular offline  RL approach which is amenable to theoretical analysis while  achieving impressive empirical performance. In addition to the dataset $\cD_{P}$, FQI uses a function class $\cF=\{f:\ScA\to [0, 1/(1-\gamma)] \}$ to approximate $Q^{*}_{P}$. The typical FQI update is given by $ f_{k+1} = \argmin_{f \in \cF}  \sum^{N}_{i=1}  (r(s_{i}, a_{i}) + \gamma \max_{b} f_{k}(s'_{i}, b) - f(s_{i}, a_{i}) )^{2}$, which aims to approximate the non-robust Bellman update using offline data with function approximation.  Under suitable assumptions, it is possible to obtain provable performance guarantees for FQI \citep{szepesvari2005finite, chen2019information, liu2020provably}.

\section{Offline Robust Reinforcement Learning}
\label{sec:offline-robust-RL-basics}

The goal of an offline robust  RL algorithm is to learn the optimal robust policy $\pi^{*}$ using a pre-collected offline dataset $\cD$. The data is typically generated according to a nominal (training) model $P^{o}$, i.e., $\cD =\{(s_i,a_i,r_{i}, s'_{i})\}_{i=1}^N$, where $(s_{i}, a_{i}) \sim \mu, \mu  \in \Delta (\ScA)$ is some data generating distribution, and $s'_{i} \sim P^{o}_{s_{i},a_{i}}$. The uncertainty set $\cP$ is defined around this nominal model $P^{o}$ as given in \eqref{eq:uncertainty-set} w.r.t. the total variation distance metric. We emphasize that the learning algorithm does not know the nominal model $P^{o}$ as it has only access to $\cD$, and hence it also does not know $\cP$. Moreover, the learning algorithm does not have data generated according to any other models in $\cP$ and has to rely only on $\cD$ to account for the behavior w.r.t. all models in $\cP$.

Learning policies for RL problems with large state-action spaces is computationally  intractable. RL algorithms typically overcome this issue by using function approximation. In this paper, we consider two function classes $\cF=\{f:\ScA\to [0, 1/(1-\gamma)] \}$ and $\cG=\{g :\ScA\to [0, 2/(\rho(1-\gamma))]\}$. We use $\cF$ to approximate $Q^{*}$ and $\cG$ to approximate the dual variable functions which we will introduce in the next section. For simplicity, we will first  assume that these function classes are finite but exponentially large, and we will use the standard log-cardinality  to characterize the sample complexity results, as given in Theorem \ref{thm:tv-guarantee}. We note that, at the cost of additional notations and analysis, infinite function classes can also be considered where the log-cardinalities are replaced by the appropriate notions of covering number.

Similar to the non-robust offline RL, we make the following standard assumptions about the data generating distribution $\mu$ and  the representation power of $\cF$.

\begin{assumption}[Concentratability] 
\label{assum-concentra-condition}
There exists a finite constant $C>0$ such that for any $\nu\in \{ d_{\pi,P^o} ~|\text{ any policy}~ \pi\} \subseteq \Delta(\ScA)$, we have $\norm{\nu/\mu}_{\infty} \leq \sqrt{C}$.
\end{assumption}

Assumption \ref{assum-concentra-condition} states that the ratio of the distribution $\nu$ and the data generating distribution $\mu$, $\nu(s,a)/\mu(s,a)$, is uniformly bounded. This assumption is widely used in the offline RL literature \citep{munos2003error,agarwal2019reinforcement,chen2019information,wang2021what,xie2021bellman} in many different forms. We borrow this assumption from \citet{chen2019information}, where they used it for non-robust offline RL. In particular, we note that the distribution $\nu$ is in the collection of discounted state-action occupancies on model $P^o$ alone for the robust RL.

\begin{assumption}[Approximate completeness] 
\label{assum-bellman-completion}
Let $\mu\in\Delta(\ScA)$ be the data  distribution. Then, $\sup_{f\in\cF} \inf_{f'\in\cF} \| f' - T f \|_{2,\mu}^2 \leq \epsilon_{\mathrm{c}}.$
\end{assumption}
Assumption \ref{assum-bellman-completion} states that the function class $\cF$ is approximately closed under the robust Bellman operator $T$. This assumption has also been widely used in the offline RL literature~\citep{agarwal2019reinforcement,chen2019information,wang2021what,xie2021bellman}.

One of the most important properties that the function class $\cF$ should have is that there must exist a function $f' \in \cF$ which well-approximates $Q^{*}$.  This assumption is typically called \textit{approximate realizability} in the offline RL literature. This is typically formalized by assuming $\inf_{f \in \cF} \norm{f - Tf}^{2}_{2, \mu} \leq \epsilon_{\mathrm{r}}$ \citep{chen2019information}. It is known that the approximate completeness assumption and the concentratability assumption imply the realizability assumption \citep{chen2019information,xie2021bellman}.

\section{Robust Fitted Q-Iteration: Algorithm and Main Results}
\label{sec:algorithm-results}

In this section, we give a step-by-step approach to overcome the challenges of the offline robust RL outlined in Section \ref{sec:introduction}. We then combine these intermediate steps to obtain our proposed RFQI algorithm. We then present our main result about the performance guarantee of the RFQI algorithm, followed by a brief description about the  proof approach.

\subsection{Dual Reformulation of Robust Bellman Operator}  
\label{sec:rob-Bellman-tv}

One key challenge in directly using the standard definition of the optimal robust value function given in \eqref{eq:robust-value-function} or of the robust Bellman operator  given in \eqref{eq:robust-bellman-primal} for developing and analyzing robust RL algorithms  is that both involve computing an expectation w.r.t. each model $P \in \cP$. Given that the data is generated only according to the nominal model $P^{o}$, estimating these expectation values is really challenging. We show that we can overcome this difficulty through the dual reformulation of the robust Bellman operator, as given below.

\begin{proposition}
\label{prop:robust-bellman-dual}
Let $M$ be an RMDP with the uncertainty set $\cP$ specified by~\eqref{eq:uncertainty-set} using the total variation distance $D(P_{s,a}, P^{o}_{s,a}) = (1/2) \|P_{s,a} - P^{o}_{s,a} \|_{1}$. Then, for any $Q: \ScA\to [0, 1/(1-\gamma)]$, the robust Bellman operator $T$ given in \eqref{eq:robust-bellman-primal} can be equivalently written as
\begin{align}
\label{eq:robust-bellman-dual}
    (T Q)(s, a) = r(s, a) - \gamma ~ \inf_{\eta\in[0, \frac{2}{\rho(1-\gamma)}]} (\E_{s' \sim P^{o}_{s,a}}[(\eta -V(s'))_+] - \eta + \rho (\eta - \inf_{s''} V(s''))_+ ), 
\end{align} 
where $V(s)=\max_{a\in\cA} Q(s,a)$. Moreover, the  inner optimization problem in \eqref{eq:robust-bellman-dual}  is convex in $\eta$.
\end{proposition}

This result mainly relies on \citet[Section 3.2]{shapiro2017distributionally} and~\citet[Proposition 1]{duchi2018learning}.
Note that in \eqref{eq:robust-bellman-dual}, the expectation is now only w.r.t.~the nominal model $P^{o}$, which opens up the possibility of using  empirical estimates obtained from the data generated according to $P^{o}$. This avoids the need to use importance sampling based techniques to account for all models in $\cP$, which often have high variance, and thus, are not desirable.

While \eqref{eq:robust-bellman-dual} provides a form that is amenable to estimation using  offline data, it involves finding $ \inf_{s''} V(s'')$. Though this computation is straightforward in a  tabular setting, it is infeasible in a function approximation setting. In order to overcome this issue, we make the following assumption. 

\begin{assumption}[Fail-state] 
\label{assum-fail-state}
The RMDP $M$ has a `fail-state' $s_{f}$, such that $r(s_{f}, a) = 0$ and $P_{s_{f}, a}(s_{f}) = 1,\;\forall a \in \cA,\;\forall P \in \cP$. 
\end{assumption}
We note that this is not a very restrictive assumption because  such a `fail-state' is quite natural in most simulated or real-world systems. For example, a state where a robot collapses and is not able to get up, either in a simulation environment like MuJoCo or in real-world setting, is such a fail state.

Assumption~\ref{assum-fail-state} immediately implies that $V_{\pi, P}(s_{f}) = 0,\;\forall P \in \cP$, and hence $V^{*}(s_{f}) = 0$ and $Q^{*}(s_{f}, a) = 0,\;\forall a \in \cA$. It is also straightforward to see that $Q_{k+1}(s_{f}, a) = 0,\;\forall a \in \cA$, where $Q_{k}$'s are the RQI iterates given by the robust Bellman update $Q_{k+1} = T Q_{k}$ with the initialization $Q_{0}=0$. By the contraction property of $T$, we have $Q_{k} \rightarrow Q^{*}$. So, under Assumption \ref{assum-fail-state}, without loss of generality, we can always keep $Q_{k}(s_{f}, a) = 0,\;\forall a \in \cA$ and for all $k$ in RQI (and later in  RFQI). So, in the light of the above description,  for the rest of the paper we will use the robust Bellman operator $T$ by setting $ \inf_{s''} V(s'') = 0$.  In particular, for any function $f : \ScA \rightarrow [0,1/(1-\gamma)]$ with $f(s_{f}, a) = 0$, the robust Bellman operator $T$  is now given by
\begin{align}
\label{eq:robust-bellman-dual-2} 
    (T f) (s, a) = r(s, a) - \gamma   \inf_{\eta\in[0, \frac{2}{(\rho(1-\gamma))}]} ~ (\E_{s' \sim P^{o}_{s,a}}[(\eta -\max_{a'}f(s',a'))_+] - (1-\rho)\eta   ).
\end{align}

\subsection{Approximately Solving the Dual Optimization using Empirical Risk Minimization}

Another key  challenge in directly using the standard definition of the optimal robust value function given in \eqref{eq:robust-value-function} or of the robust Bellman operator  given in \eqref{eq:robust-bellman-primal} for developing and analyzing robust RL algorithms  is that both involve an optimization over $\cP$. The dual reformulation given in \eqref{eq:robust-bellman-dual-2} partially overcomes this challenge also, as the optimization over $\cP$ is now  replaced by a {convex} optimization over a scalar $\eta\in[0, 2/(\rho (1-\gamma))]$. However, this still requires solving an optimization for each $(s, a)\in\ScA$, which is  clearly infeasible even for moderately sized state-action spaces, not to mention the function approximation setting. Our key idea to overcome this difficulty is to reformulate this as a functional optimization problem instead of solving it as multiple scalar optimization problems. This functional optimization method will make it amenable to approximately solving the dual problem using an empirical risk minimization approach with offline data.

Consider the probability (measure) space $(\ScA,\Sigma(\ScA),\mu)$ and let $L^{1}(\ScA,\Sigma(\ScA),\mu)$ be the set of all absolutely integrable functions defined on this space.\footnote{In the following, we will simply denote $L^{1}(\ScA,\Sigma(\ScA),\mu)$ as $L^{1}$ for conciseness. } In other words, $L^{1}$ is the set of all functions $g : \ScA  \rightarrow  \cC \subset \R$, such that $\norm{g}_{1,\mu}$ is finite. We set $\mathcal{C} = [0, 2/\rho (1 - \gamma)]$, anticipating the solution of the dual optimization problem~\eqref{eq:robust-bellman-dual-2}. We also note $\mu$ is the data generating distribution which is a $\sigma$-finite measure. 

For any given function $f : \ScA \rightarrow [0,1/(1-\gamma)]$, we define the loss function $L_{\mathrm{dual}}(\cdot; f)$ as 
\begin{align}
    \label{L-f-eta-loss}
    L_{\mathrm{dual}}(g; f) = \E_{s,a\sim\mu} [  \E_{s'\sim P^o_{s,a}} [(g(s,a) - \max_{a'} f(s',a'))_+]  - (1-\rho) g(s,a)], \quad \forall g\in L^1.
\end{align}
In the following lemma, we show  that the scalar optimization over $\eta$ for each $(s,a)$ pair in \eqref{eq:robust-bellman-dual-2} can be replaced by a single functional optimization w.r.t.~the loss function $L_{\mathrm{dual}}$. 

\begin{lemma}
\label{lem:interchanging-lemma-1}
Let $L_{\mathrm{dual}}$ be the loss function defined in \eqref{L-f-eta-loss}. Then, for any function $f : \ScA \rightarrow [0,1/(1-\gamma)]$, we have
\begin{align}
\label{eq:interchanging-lemma-1}
    \inf_{g \in L^{1}}  L_{\mathrm{dual}}(g; f)  = \E_{s,a\sim\mu} \Big[ \inf_{\eta\in[0, \frac{2}{(\rho(1-\gamma))}]} \Big(\E_{s'\sim P^o_{s,a}} \big[\big(\eta - \max_{a'} f(s',a')\big)_+\big]  - (1-\rho)\eta\Big)\Big].
\end{align}
\end{lemma}
Note that the RHS of \eqref{eq:interchanging-lemma-1} has minimization over $\eta$ for each $(s, a)$ pair and  minimization is inside the expectation $\E_{s,a\sim\mu}[\cdot]$. However, the LHS of \eqref{eq:interchanging-lemma-1} has a single functional minimization over $g \in L^{1}$ and this minimization is outside the expectation. For interchanging the expectation and minimization, and for moving from point-wise optimization to functional optimization, we use the result from \citet[Theorem 14.60]{rockafellar2009variational}, along with the fact that $L^{1}$ is a decomposable space. We also note that this result has been used in many  recent works on distributionally robust optimization  \citep{shapiro2017distributionally,duchi2018learning} (see Appendix \ref{appendix:useful-tech-results} for more details). 

We can now define the empirical loss function $\widehat{L}_{\mathrm{dual}}$ corresponding to the true loss $L_{\mathrm{dual}}$ as
\begin{align} \label{eq:L-dual-empirical-loss}
    \widehat{L}_{\mathrm{dual}}(g;f) = \frac{1}{N} \sum^{N}_{i=1} (g(s_{i},a_{i}) - \max_{a'} f(s'_{i},a'))_+  - (1-\rho) g(s_{i},a_{i}).
\end{align}
Now, for any given $f$, we can  find an approximately optimal dual function through the \textit{empirical risk minimization} approach as  $\inf_{g \in L^{1}} \widehat{L}_{\mathrm{dual}}(g;f)$.

As we mentioned in Section \ref{sec:offline-robust-RL-basics}, our offline robust RL algorithm is given an input function class  $\cG=\{g :\ScA\to [0, 2/(\rho(1-\gamma))]\}$  to approximate the dual variable functions. So, in the empirical risk minimization, instead of taking the infimum over all the functions in $L^{1}$, we can only take the infimum over all the functions in $\cG$. For this to be meaningful, $\cG$ should have sufficient representation power. In particular, the result in Lemma \ref{lem:interchanging-lemma-1} should hold approximately even if we replace the infimum over $L^{1}$ with infimum over $\cG$. 
One can see that this is similar to the realizability requirement for the function class $\cF$ as described in Section \ref{sec:offline-robust-RL-basics}. We formalize the representation power of $\cG$ in the following assumption. 

\begin{assumption}[Approximate dual realizability] \label{assum-dual-realizability}
For all $f \in \cF$, there exists a uniform constant $\epsilon_{\textrm{dual}}$ such that $ \inf_{g \in \cG}  L_{\mathrm{dual}}(g; f) - \inf_{g \in L^{1}}   L_{\mathrm{dual}}(g; f) \leq \epsilon_{\textrm{dual}}$
\end{assumption}

Using the above assumption, for any given $f\in\mathcal F$, we can  find an approximately optimal dual function $\widehat{g}_{f} \in \cG$  through the \textit{empirical risk minimization} approach as  $\widehat{g}_{f} = \argmin_{g \in \cG} \widehat{L}_{\mathrm{dual}}(g;f)$. 

In order to characterize the performance of this approach, consider the operator $T_g$ for any $g\in\cG$ as 
\begin{align}
\label{eq:Tg}
    (T_g f) (s, a) = r(s, a) - \gamma (\E_{s' \sim P^{o}_{s,a}}[(g(s,a) -\max_{a'}f(s',a'))_+] - (1-\rho)g(s,a)   ),
\end{align}  for all  $f\in\cF$ and $(s,a)\in\ScA$. We will show in Lemma \ref{lem:erm-high-prob-bound} in Appendix \ref{appendix:thm:tv-guarantee} that the error $\sup_{f\in\cF} \|T f -  T_{\widehat{g}_f} f\|_{1,\mu}$ is $\cO(\log(|\cF|/\delta)/\sqrt{N})$ with probability at least $1-\delta$.

\subsection{Robust Fitted Q-iteration}
\label{sec:rofqi-algorithm}

The intuitive idea behind our robust fitted Q-iteration (RFQI) algorithm is to approximate the exact RQI update step $Q_{k+1} = T Q_{k}$  with function approximation using offline data. The exact RQI step requires updating each $(s,a)$-pair separately, which is not scalable to large state-action spaces. So, this is replaced by the function approximation as $Q_{k+1} = \argmin_{f \in \cF} \norm{T Q_{k} - f}^{2}_{2,\mu}$. It is still infeasible to perform this update as it requires to exactly compute the expectation (w.r.t. $P^{o}$ and $\mu$) and to solve the dual problem accurately. We overcome these issues by replacing both these exact computations with empirical estimates using the offline data. We note that this intuitive idea is similar to that of the FQI algorithm in the non-robust case. However, RFQI has unique challenges due to the nature of the robust Bellman operator $T$ and the presence of the dual optimization problem within $T$. 

Given a dataset $\cD$, we also follow the standard non-robust offline RL choice of least-squares residual minimization \citep{chen2019information,xie2021bellman,wang2021what}. Define the empirical loss of $f$ given $f'$ (which represents the $Q$-function from the last iteration) and dual variable function $g$ as
\begin{align}
\label{eq:loss-rfqi-1}
     \widehat{L}_{\mathrm{RFQI}}(f;f', g) =  \frac{1}{N} \sum^{N}_{i=1} \bigg( \begin{array}{ll} r(s_i, a_i) + \gamma    \big(-(g(s_{i},a_{i}) -\max_{a'}f'(s'_{i},a'))_+ \\ \hspace{2cm}+ \;(1-\rho) g(s_{i},a_{i})\big)  - f(s_{i}, a_{i}) \end{array}    \bigg)^{2}.
\end{align} 
The correct dual variable function to be used in \eqref{eq:loss-rfqi-1} is the optimal dual variable  $g^{*}_{f'} = \argmin_{g \in \cG} L_{\mathrm{dual}}(g; f')$ corresponding to the last iterate $f'$, which we will approximate it by  $\widehat{g}_{f'} = \argmin_{g \in \cG} \widehat{L}_{\mathrm{dual}}(g; f')$. The  RFQI update is then obtained as  $\argmin_{f \in \cF} \widehat{L}_{\mathrm{RFQI}}(f;f',  \widehat{g}_{f'})$. 

Summarizing the individual steps described above, we formally give our RFQI algorithm below. 

\begin{algorithm}[H]
	\caption{Robust Fitted Q-Iteration (RFQI) Algorithm}	
	\label{alg:RFQI-Algorithm}
	\begin{algorithmic}[1]
		\STATE \textbf{Input:} Offline dataset $\cD =(s_i,a_i,r_{i}, s'_{i})_{i=1}^N$, function classes $\cF$ and $\cG$. 
		\STATE \textbf{Initialize:} $Q_{0}\equiv 0 \in \cF$.
		\FOR {$k=0,\cdots,K-1$ } 
        \STATE \textbf{Dual variable function optimization:} Compute the dual variable function corresponding to $Q_{k}$ through empirical risk minimization as $g_{k} = \widehat{g}_{Q_{k}} = \argmin_{g \in \cG} \widehat{L}_{\mathrm{dual}}(g; Q_{k})\;\;$ (see \eqref{eq:L-dual-empirical-loss}).
        \STATE \textbf{Robust Q-update:} Compute the next iterate $Q_{k+1}$ through least-squares regression as
        $\quad Q_{k+1}=\argmin_{Q\in\cF}\widehat{L}_{\mathrm{RFQI}}(Q;Q_{k}, g_{k})\;\;$ (see \eqref{eq:loss-rfqi-1}).
		\ENDFOR
		
		\STATE \textbf{Output:} $\pi_{K}  = \argmax_a Q_{K}(s,a)$
	\end{algorithmic}
\end{algorithm}

Now we state our main theoretical result on the performance of the RFQI algorithm. 

\begin{theorem}
\label{thm:tv-guarantee}
Let Assumptions \ref{assum-concentra-condition}-\ref{assum-dual-realizability} hold. Let $\pi_{K}$ be the output of the RFQI algorithm after $K$ iterations. Denote $J^{\pi}=\E_{s\sim d_0}[V^{\pi}(s)]$ where $d_0$ is initial state distribution. Then, for any $\delta\in(0,1)$, with probability at least  $1 - 2 \delta$,  we have 
\begin{align*}
	  J^{\pi^*} - J^{\pi_{K}} \leq \frac{\gamma^K}{(1-\gamma)^2} &+ \frac{\sqrt{C}(\sqrt{6\epsilon_{\textrm{c}}} + \gamma \epsilon_{\textrm{dual}}) }{(1-\gamma)^2}  +   \frac{16}{\rho(1-\gamma)^3} \sqrt{\frac{18 C \log(2|\cF||\cG|/\delta)}{N}}. \end{align*}
\end{theorem}

\begin{remark}
Theorem \ref{thm:tv-guarantee} states that the RFQI algorithm can achieve approximate optimality. To see this, note that with $K\geq\cO(\frac{1}{ \log(1/\gamma)}  \log(\frac{1}{\epsilon (1-\gamma)}))$, and neglecting the second term corresponding to (inevitable) approximation errors  $\epsilon_{\textrm{c}}$ and  $\epsilon_{\textrm{dual}}$, we get $J^{\pi^*} - J^{\pi_{K}} \leq {\epsilon}/{(1-\gamma)}$ with probability greater than $1-2\delta$ for any $\epsilon, \delta \in (0,1)$, as long as the number of samples  $N\geq\cO(\frac{1}{(\rho\epsilon)^{2} (1-\gamma)^{4}} \log \frac{|\cF| |\cG|}{\delta })$. So, the above theorem can also be interpreted as a \textbf{sample complexity} result. 
\end{remark}

\begin{remark}
The known sample complexity of robust-RL in the tabular setting is $\widetilde{\cO}(\frac{|\cS|^2 |\cA|}{(\rho\epsilon)^{2} (1-\gamma)^{4}})$ \citep{yang2021towards,panaganti22a}.  Considering $\widetilde{\cO}(\log(|\cF| | \cG|))$ to be $\widetilde{\cO}(|\cS| |\cA|)$, we  can recover the same bound as in the tabular setting (we save $|\cS|$  due to the use of Bernstein inequality). 
\end{remark}

\begin{remark}
Under similar Bellman completeness and concentratability assumptions, RFQI sample complexity is comparable to that of a non-robust offline RL algorithm, i.e.,~$\cO(\frac{1}{\epsilon^{2} (1-\gamma)^{4}} \log \frac{|\cF|}{\delta })$~\citep{chen2019information}. As a consequence of robustness, we have  $\rho^{-2}$ and $\log(|\cG|)$ factors in our bound. 
\end{remark}

\subsection{Proof Sketch} \label{sec:proof-sketch}
Here  we briefly explain the key ideas used in the analysis of RFQI for obtaining the optimality gap bound in Theorem \ref{thm:tv-guarantee}. The complete proof is provided in Appendix \ref{appendix:thm:tv-guarantee}.

\textit{Step 1:} To bound $J^{\pi^*} - J^{\pi_K}$, we  connect it to the error  $\|Q^{\pi^*} -  Q_{K}\|_{1,\nu }$  for any state-action distribution $\nu$.  While the similar step follows almost immediately using the well-known performance lemma in the analysis of non-robust FQI, such a result is not known in the robust RL setting. So, we derive the basic inequalities to get a recursive form and to  obtain the bound $J^{\pi^*} - J^{\pi_K} \leq 2\|Q^{\pi^*} -  Q_{K}\|_{1,\nu }/(1-\gamma)$ (see~\eqref{eq:thm-bound-part-1} and the steps before in Appendix~\ref{appendix:thm:tv-guarantee}). \\
\textit{Step 2:} To bound $\|Q^{\pi^*} -  Q_{K}\|_{1,\nu }$ for any state-action distribution $\nu$ such that $\norm{\nu/\mu}_{\infty}\leq\sqrt{C}$, we decompose it to get a recursion,  with  approximation terms based on the least-squares regression and empirical risk minimization.  Recall that   $\widehat{g}_f$ is the dual variable function from the algorithm for state-action value function $f\in\cF$. Denote $\widehat{f}_{g}$ as the least squares solution from the algorithm for the state-action value function $f\in\cF$ and dual variable function $g\in\cG$, i.e., $\widehat{f}_{g}=\argmin_{Q\in\cF}\widehat{L}_{\mathrm{RFQI}}(Q;f, g)$. By recursive use of the obtained inequality  \eqref{eq:pf-sketech-stp4-1} (see Appendix  \ref{appendix:thm:tv-guarantee}) and  using uniform bound,  we  get 
\begin{align*}
    \|Q^{\pi^*} -  Q_{K}\|_{1,\nu } \leq \frac{\gamma^K }{1-\gamma} + \frac{\sqrt{C} }{1-\gamma} \sup_{f\in\cF} \|T f -  T_{\widehat{g}_f} f\|_{1,\mu}
+ \frac{\sqrt{C} }{1-\gamma} \sup_{f\in\cF} \sup_{g\in\cG} \| T_{g} f -  \widehat{f}_{g}\|_{2,\mu}.
\end{align*}
\textit{Step 3:} We recognize that $\sup_{f\in\cF} \|T f -  T_{\widehat{g}_f} f\|_{1,\mu}$ is an empirical risk minimization error term.  Using Rademacher
complexity based bounds, we show in Lemma~\ref{lem:erm-high-prob-bound} that this error is $\cO(\log(|\cF|/\delta)/\sqrt{N})$ with high probability.  \\
\textit{Step 4:} Similarly, we also recognize that $ \sup_{f\in\cF} \sup_{g\in\cG} \| T_{g} f -  \widehat{f}_{g}\|_{2,\mu}$ is a least-squares regression error term. We also show that this error is $\cO(\log(|\cF||\cG|/\delta)/\sqrt{N})$ with high probability. We adapt the generalized least squares regression result to accommodate the modified target functions resulting from the robust Bellman operator to obtain this bound (see Lemma~\ref{lem:least-squares-generalization-bound}). \\
The proof is complete after combining steps 1-4 above.

\section{Experiments}
\label{sec:main-experiments}

\begin{figure*}[h]
	\centering
	\begin{minipage}{.32\textwidth}
		\centering
		\includegraphics[width=\linewidth]{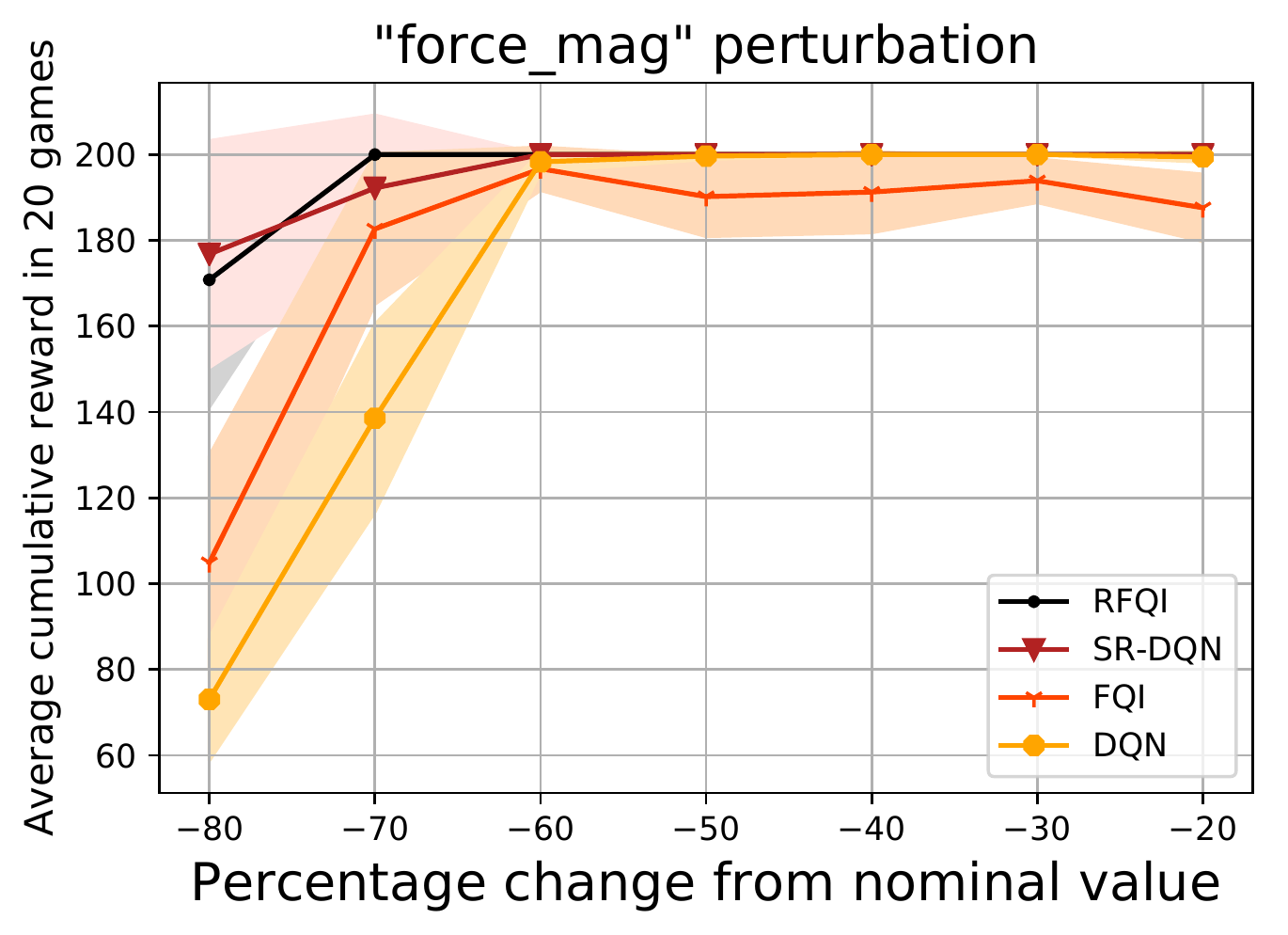}
		\captionof{figure}{CartPole}
		\label{fig:cart_force-in}
	\end{minipage}
	\begin{minipage}{.32\textwidth}
		\centering
		\includegraphics[width=\linewidth]{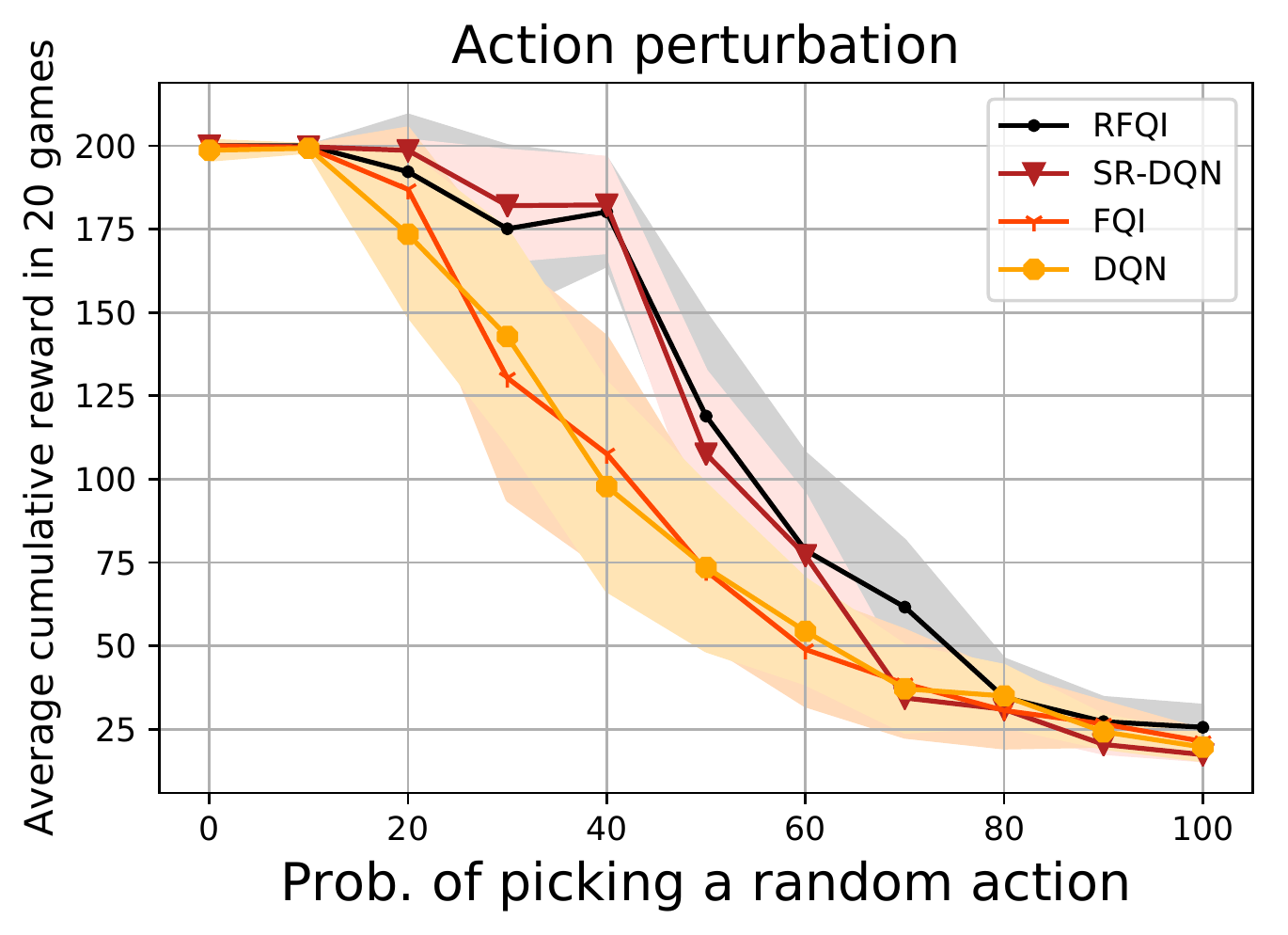}
		\captionof{figure}{CartPole}
		\label{fig:cart_action-in}
	\end{minipage}
	\begin{minipage}{.32\textwidth}
		\centering
		\includegraphics[width=\linewidth]{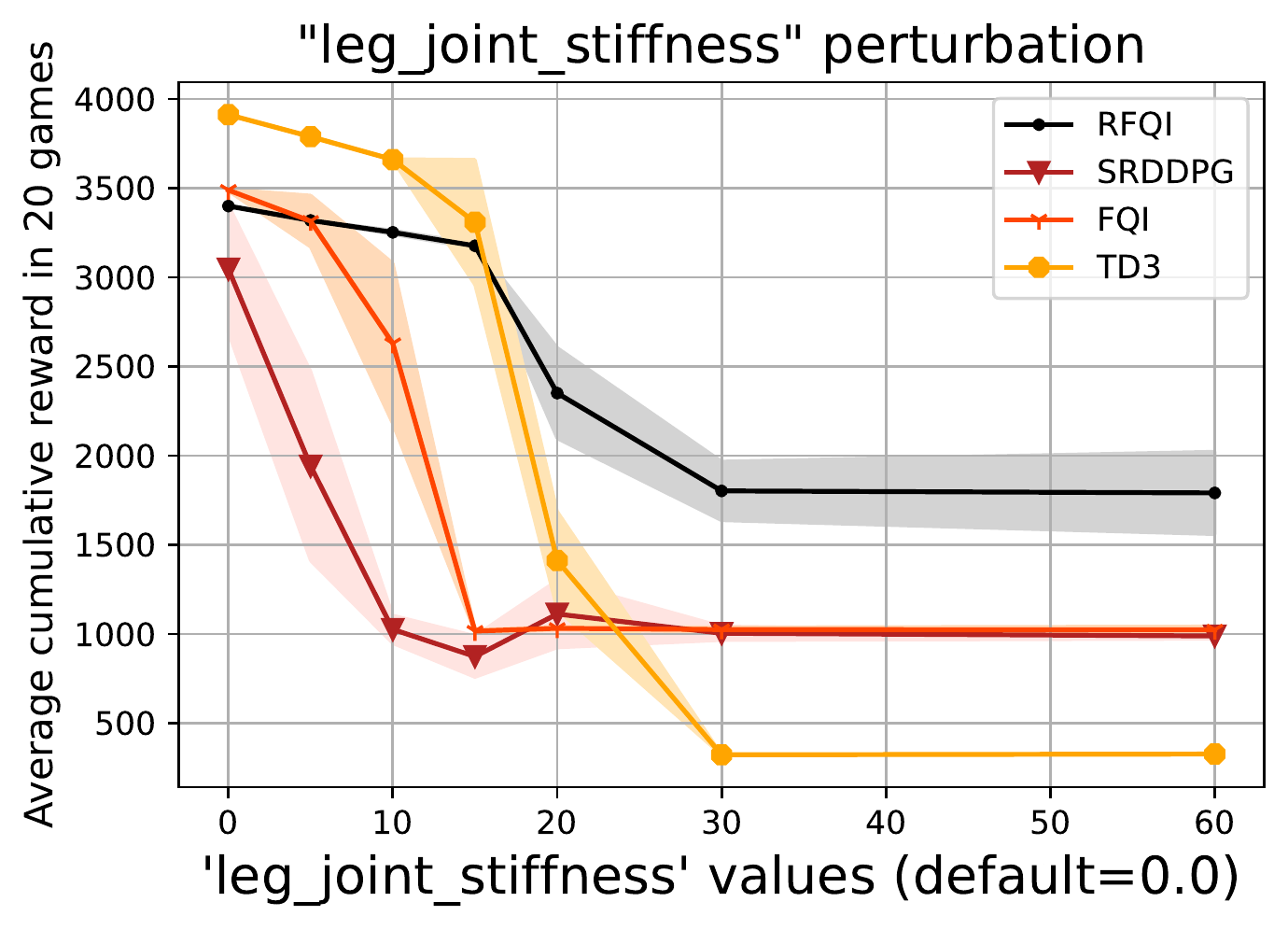}
		\captionof{figure}{Hopper}
		\label{fig:hopper_leg}
	\end{minipage}
\end{figure*}

\looseness=-1 Here, we demonstrate the robust performance of our RFQI algorithm by evaluating it on  \textit{Cartpole} and \textit{Hopper} environments in OpenAI Gym \citep{brockman2016openai}. In all the figures shown, the quantity in the vertical axis is averaged over $20$ different seeded runs depicted by the  thick line and the band around it is the $\pm 0.5$ standard deviation. \textit{A more detailed description of the experiments, and results on additional experiments, are deferred to Appendix \ref{appendix:simulations}.} We provide our code in \textbf{github webpage} \url{https://github.com/zaiyan-x/RFQI} containing instructions to reproduce all results in this paper.

\looseness=-1 For the \textit{Cartpole}, we compare RFQI algorithm against the non-robust RL algorithms FQI and DQN, and the soft-robust RL algorithm proposed in \cite{derman2018soft}.  We test the robustness of the algorithms by changing the parameter \textit{force\_mag} (to model external force disturbance),  and also by introducing action perturbations (to model actuator noise). Fig. \ref{fig:cart_force-in} and  Fig. \ref{fig:cart_action-in} shows superior robust performance of RFQI compared to the non-robust FQI and DQN. The RFQI performance is similar to that of  soft-robust DQN. We note that soft-robust RL algorithm (here soft-robust DQN) is an online deep RL algorithm (and not an offline RL algorithm) and has no provable performance guarantee. Moreover,  soft-robust RL algorithm requires generating  online data according a number of models in the uncertainty set, whereas RFQI only requires offline  data according to a single nominal training model. 

For the \textit{Hopper}, we compare RFQI algorithm against the non-robust RL algorithms FQI and TD3 \citep{fujimoto2018addressing}, and the soft-robust RL (here soft-robust DDPG) algorithm proposed in \cite{derman2018soft}. We test the robustness of the algorithms by changing the parameter \textit{leg\_joint\_stiffness}. Fig. \ref{fig:hopper_leg} shows the superior performance of our RFQI algorithm against the non-robust algorithms and soft-robust DDPG algorithm. The  average episodic reward of RFQI remains almost the same initially, and later decays much less and gracefully   when compared to the non-robust FQI and TD3.

\section{Conclusion}

In this work, we presented a novel robust RL algorithm called Robust Fitted Q-Iteration algorithm with provably optimal performance for an RMDP with arbitrarily large state space, using only offline data with function approximation. We also demonstrated the  superior performance of the proposed algorithm on standard benchmark problems.

\looseness=-1 One limitation of our present work is that, we considered only the uncertainty set defined with respect to the total variation distance. In future work, we will consider uncertainty sets defined with respect to other $f$-divergences such as KL-divergence and Chi-square divergence.  Finding a lower bound for the sample complexity and relaxing the assumptions used are also important and challenging problems. 

\section{Acknowledgements}

This work was supported in part by the National Science Foundation (NSF) grants NSF-CAREER- EPCN-2045783 and NSF ECCS 2038963. Any opinions, findings, and conclusions or recommendations expressed in this material are those of the authors and do not necessarily reflect the views of the sponsoring agencies.

\bibliography{References-RFQI}

\newpage
\section*{Checklist}

\begin{enumerate}

\item For all authors...
\begin{enumerate}
  \item Do the main claims made in the abstract and introduction accurately reflect the paper's contributions and scope?
    \answerYes{See \textbf{contributions} in the Introduction.}
  \item Did you describe the limitations of your work?
    \answerYes{The discussions on the assumptions describes the limitations.}
  \item Did you discuss any potential negative societal impacts of your work?
    \answerNA{}
  \item Have you read the ethics review guidelines and ensured that your paper conforms to them?
    \answerYes{}
\end{enumerate}

\item If you are including theoretical results...
\begin{enumerate}
  \item Did you state the full set of assumptions of all theoretical results?
    \answerYes{See Sections \ref{sec:offline-robust-RL-basics}-\ref{sec:rofqi-algorithm} }
        \item Did you include complete proofs of all theoretical results?
    \answerYes{We provide proof sketch \ref{sec:proof-sketch} in main paper and the complete proof in Appendix with self-contained material.}
\end{enumerate}

\item If you ran experiments...
\begin{enumerate}
  \item Did you include the code, data, and instructions needed to reproduce the main experimental results (either in the supplemental material or as a URL)?
    \answerYes{}
  \item Did you specify all the training details (e.g., data splits, hyperparameters, how they were chosen)?
    \answerYes{Described in the Appendix.}
        \item Did you report error bars (e.g., with respect to the random seed after running experiments multiple times)?
    \answerYes{Described in the main paper and the Appendix.}
        \item Did you include the total amount of compute and the type of resources used (e.g., type of GPUs, internal cluster, or cloud provider)?
    \answerYes{Mentioned in the Appendix.}
\end{enumerate}

\item If you are using existing assets (e.g., code, data, models) or curating/releasing new assets...
\begin{enumerate}
  \item If your work uses existing assets, did you cite the creators?
    \answerYes{}
  \item Did you mention the license of the assets?
    \answerYes{}
  \item Did you include any new assets either in the supplemental material or as a URL?
    \answerNA{}
  \item Did you discuss whether and how consent was obtained from people whose data you're using/curating?
    \answerNA{}
  \item Did you discuss whether the data you are using/curating contains personally identifiable information or offensive content?
    \answerNA{}
\end{enumerate}

\item If you used crowdsourcing or conducted research with human subjects...
\begin{enumerate}
  \item Did you include the full text of instructions given to participants and screenshots, if applicable?
    \answerNA{}
  \item Did you describe any potential participant risks, with links to Institutional Review Board (IRB) approvals, if applicable?
    \answerNA{}
  \item Did you include the estimated hourly wage paid to participants and the total amount spent on participant compensation?
    \answerNA{}
\end{enumerate}

\end{enumerate}

\newpage
\appendix

\section*{Appendix}

\section{Useful Technical Results}
\label{appendix:useful-tech-results}
In this section, we state some existing results from concentration inequalities, generalization bounds, and optimization theory that we will use later in our analysis. We first state the Berstein's inequality that utilizes second-moment to get a tighter concentration inequality.

\begin{lemma}[Bernstein's inequality \hspace{-0.1cm} \text{\citep[Theorem 2.8.4]{vershynin2018high}}] \label{lem:bernstein-ineq}
Let $X_1,\cdots, X_T$ be independent random variables. Assume that $|X_t - \E[X_t]| \leq M$, for all $t$. Then, for any $\epsilon > 0$, we have 
\[ 
\pr \left( \Big|\frac{1}{T} \sum_{t=1}^T (X_t - \E[X_t])\Big| \geq \epsilon \right) \leq 2\exp\left( -\frac{T^2\epsilon^2}{2\sigma^2 + \frac{2MT\epsilon}{3}} \right), 
\] 
where $\sigma^2 = \sum_{t=1}^T \E[X_t^2]$. Furthermore, if $X_1,\cdots, X_T$ are independent and identically distributed random variables, then   for any $\delta\in(0,1)$, we have 
\[ 
\Big|\E[X_1] - \frac{1}{T} \sum_{t=1}^T X_t \Big| \leq \sqrt{\frac{2 \E[X_1^2] \log(2/\delta)}{T}} +\frac{M\log(2/\delta)}{3T},
\] 
with probability at least $1-\delta$.
\end{lemma}

We now state a result 
for the generalization bounds on empirical risk minimization (ERM) problems. This result is adapted from \citet[Theorem 26.5, Lemma 26.8, Lemma 26.9]{shalev2014understanding}.

\begin{lemma}[ERM generalization bound] \label{lem:erm-generalization-bound}
Let $P$ be the data generating distribution on the space $\cX$ and let $\cH$ be a given hypothesis class of functions. Assume that for all $x\in\cX$ and $h \in \cH$ we have that $|l(h,x)| \leq c_1$ for some positive constant $c_1>0$. Given a dataset $\cD=\{X_i\}_{i=1}^N$, generated independently from $P$, denote $\hat{h}$ as the ERM solution, i.e. $\hat{h}=\argmin_{h\in\cH} (1/N) \sum_{i=1}^N l(h, X_i)$. For any fixed $\delta\in(0,1)$ and $h^*\in\argmin_{h\in\cH} \E_{X\sim P}[l(h,X)]$, we have 
\begin{align} \label{eq:erm-generalization-bound-part-1} \E_{X\sim P}[l(\hat{h},X)] - \E_{X\sim P}[l(h^*,X)] \leq 2 R(l\circ \cH \circ \cD) + 5 c_1 \sqrt{\frac{2\log(8/\delta)}{N}}, 
\end{align} 
with probability at least $1-\delta$, where $R(\cdot)$ is the Rademacher complexity of $l\circ \cH$ given by 
\[ 
R(l\circ \cH \circ \cD) = \frac{1}{N} \E_{\{\sigma_i\}_{i=1}^N} \left( \sup_{g\,\in\, l\circ \cH}~ \sum_{i=1}^N \sigma_i g(X_i) \right), 
\]
in which $\sigma_i$'s are independent from $X_i$'s and are independently and identically distributed according to the Rademacher random variable $\sigma$, i.e. $\pr(\sigma=1) = 0.5 = \pr(\sigma=-1)$.

Furthermore, if $\cH$ is a finite hypothesis class, i.e. $|\cH|<\infty$,  with $|h\circ x| \leq c_2$ for all $h\in\cH$ and $x\in\cX$, and $l(h,x)$ is $c_3$-Lipschitz in $h$, then we have 
\begin{align} \label{eq:erm-generalization-bound-part-2}
    \E_{X\sim P}[l(\hat{h},X)] - \E_{X\sim P}[l(h^*,X)] \leq 2 c_2 c_3 \sqrt{\frac{2\log(|\cH|)}{N}} + 5 c_1 \sqrt{\frac{2\log(8/\delta)}{N}},
\end{align}  
with probability at least $1-\delta$.
\end{lemma}

We now mention two important concepts from variational analysis \citep{rockafellar2009variational} literature that is useful to relate minimization of integrals and the integrals of pointwise minimization under special class of functions.

\begin{definition}[Decomposable spaces and Normal integrands \text{\citep[Definition 14.59, Example 14.29]{rockafellar2009variational}}] \label{def:decomposable-normal-integrand}
A space $\cX$ of measurable functions is a {decomposable space} relative to an underlying measure space $(\Omega, \cA, \mu)$, if for every function $x_0 \in \cX$, every set $A \in \cA$ with $\mu(A) < \infty$, and any bounded measurable function $x_1\,:\,A\to\R$, the function $x(\omega) = x_0(\omega)\mathbbm{1}(\omega \notin A) + x_1(\omega)\mathbbm{1}(\omega \in A)$ belongs to $\cX$. A function $f : \Omega \times \R \to \R$ (finite-valued) is a {normal integrand}, if and only if $f(\omega,x)$ is $\cA$-measurable in $\omega$ for each $x$ and is continuous in $x$ for each $\omega$.
\end{definition}

\begin{remark} \label{remark:decomposable}
A few examples of decomposable spaces are $L^p(\ScA,\Sigma(\ScA),\mu)$ for any $p\geq 1$ and $\cM(\ScA,\Sigma(\ScA))$, the space of all $\Sigma(\ScA)$-measurable functions. \end{remark}

\begin{lemma}[\text{\citealp[Theorem 14.60]{rockafellar2009variational}}] \label{lem:integral-min-exchange-rockafellar}
Let $\cX$ be a space of measurable functions from $\Omega$ to $\R$ that is decomposable relative to a $\sigma$-finite measure $\mu$ on the $\sigma$-algebra $\cA$. Let $f : \Omega \times \R \to \R$ (finite-valued) be a normal integrand. Then, we have 
\[ 
\inf_{x\in\cX} \int_{\omega\in\Omega} f(\omega,x(\omega)) \mu(\dd\omega) = \int_{\omega\in\Omega} \left( \inf_{x\in\R} f(\omega,x) \right) \mu(\dd\omega). 
\]
Moreover, as long as the above infimum is not $-\infty$, we have that
\begin{align*}
 x' \in \argmin_{x\in\cX} \int_{\omega\in\Omega} f(\omega,x(\omega)) \mu(\dd\omega), 
 \end{align*}
if and only if $\;x'(\omega) \in \argmin_{x\in\R} f(\omega,x)\cdot\mu\;$ almost surely.
\end{lemma}

We now give one result from distributioanlly robust optimization. The $f$-divergence between the distributions $P$ and $P^o$ is defined as 
\begin{align}
    \label{eq:f-divergence-defn}
    D_{f}(P \| P^o) = \int f(\frac{\dd P}{\dd P^o})\dd P^o,
\end{align}
where $f$ is a convex function \citep{csiszar1967information,moses2011further}. We obtain different divergences for different forms of the function $f$, including some well-known divergences. For example,  
$f(t) = |t-1|/2$ gives Total Variation (TV), $f(t) = t\log t$ gives  Kullback-Liebler (KL), $f(t) = (t-1)^2$ gives Chi-square, and $f(t) = (\sqrt{t}-1)^2$ gives squared Hellinger divergences.

Let $P^o$ be a distribution on the space $\cX$ and let $l: \cX \to \R$ be a loss function.  We have the following result from the \textit{distributionally robust optimization} literature, see e.g.,~\citet[Section 3.2]{shapiro2017distributionally} and~\citet[Proposition 1]{duchi2018learning}.

\begin{proposition}
\label{prop:dro-inner-problem-solution}
Let $D_{f}$ be the $f$-divergence as defined in \eqref{eq:f-divergence-defn}. Then, 
\begin{align}
\label{eq:dro-inner-problem-solution}
   \sup_{D_f(P\|P^o) \leq \rho} \E_{P}[l(X)] = \inf_{\lambda>0, \eta\in\R} ~~ \E_{P^o} \left[ \lambda f^* \left(\frac{l(X)-\eta}{\lambda}\right)\right] + \lambda\rho + \eta,
\end{align}
where $f^*(s) = \sup_{t\geq 0} \{st - f(t)\}$ is the Fenchel conjugate.
\end{proposition}

Note that on the right hand side of \eqref{eq:dro-inner-problem-solution}, the expectation is taken only with respect to $P^{o}$. We will use the above result to derive the dual reformulation of the robust Bellman operator.

\section{Proof of the Proposition \ref{prop:robust-bellman-dual}}
\label{appendix:prop:robust-bellman-dual}

As the first step, we adapt the result given in Proposition~\ref{prop:dro-inner-problem-solution} in two ways: $(i)$ Since Proposition~\ref{prop:robust-bellman-dual} considers the TV uncertainty set, we will derive  the specific form of this result for the TV uncertainty set, $(ii)$ Since Proposition~\ref{prop:robust-bellman-dual} considers the minimization problem instead of the maximization problem, unlike in Proposition~\ref{prop:dro-inner-problem-solution}, we will derive the specific form of this result for minimization.

\begin{lemma} 
\label{lem:TV-inner-problem}
Let $D_{f}$ be as defined in \eqref{eq:f-divergence-defn} with $f(t)=|t-1|/2$ corresponding to the TV uncertainty set. Then, 
\begin{align*}
      \inf_{D_f(P\|P^o) \leq \rho} \E_{P}[l(X)] =  -\inf_{\eta \in\R}\;\E_{P^o}[(\eta-l(X))_+] + (\eta - \inf_{x\in\cX} l(x))_+ \times \rho - \eta, 
\end{align*}
\end{lemma}

\begin{proof}
First, we will compute the Fenchel conjugate of $f(t)=|t-1|/2$. We have 
\begin{align*}
    f^*(s) &= \sup_{t\geq 0} ~ \{st - \frac{1}{2}|t-1|\}
    = \max \big\{\!\!\sup_{t\in[0,1]} \!\{(s+\frac{1}{2})t - \frac{1}{2}\} \;,\; \sup_{t> 1} \!\{(s-\frac{1}{2}) t +\frac{1}{2}\} \big\}.
\end{align*}
It is easy to see that for $s>1/2$, we have $f^*(s)=+\infty$, and for $s\leq -1/2$, we have $f^*(s)=-1/2$. For $s \in [-1/2,1/2]$, we have
\begin{align*}
    f^*(s) &= \max \big\{ \sup_{t\in[0,1]}\{(s+\frac{1}{2})t - \frac{1}{2}\} \;,\; \sup_{t> 1} (\{(s-\frac{1}{2}) t +\frac{1}{2}\} \big\} \\
    &= \max \big\{  ((s+\frac{1}{2})\cdot 1 - \frac{1}{2}),~  ((s-\frac{1}{2})\cdot 1 +\frac{1}{2}) \big\} = s.
\end{align*} 
Thus, we have 
\begin{equation*}
  f^*(s) =
    \begin{cases}
      -\frac{1}{2} & s\leq -\frac{1}{2},\\
      s & s\in [-\frac{1}{2},\frac{1}{2}]\\
      +\infty & s > \frac{1}{2}.
    \end{cases}  .  
\end{equation*}

From Proposition \ref{prop:dro-inner-problem-solution}, we obtain
\begin{align*}
     \sup_{D_f(P\|P^o) \leq \rho} \E_{P}[l(X)]   &= \inf_{\lambda>0, \eta\in\R} ~ \E_{P^o} [\lambda f^*(\frac{l(X)-\eta}{\lambda})] + \lambda\rho + \eta  \\
    &= \inf_{\lambda,\eta: \lambda>0, \eta\in\R, \frac{\sup_{x\in\cX}l(x)-\eta}{\lambda} \leq \frac{1}{2}} ~  \E_{P^o} [\lambda \max\{ \frac{l(X)-\eta}{\lambda}, -\frac{1}{2} \}] + \lambda\rho + \eta  \\
    &= \inf_{\lambda,\eta: \lambda>0, \eta\in\R, \frac{\sup_{x\in\cX}l(x)-\eta}{\lambda} \leq \frac{1}{2}} ~ \E_{P^o} [\max\{ l(X)-\eta, -\lambda/2 \}] + \lambda\rho + \eta  \\
    &= \inf_{\lambda,\eta: \lambda>0, \eta\in\R, \frac{\sup_{x\in\cX}l(x)-\eta}{\lambda} \leq \frac{1}{2}} ~ \E_{P^o} [ (l(X)-\eta +\lambda/2)_+ ] - \lambda/2 +  \lambda\rho + \eta  \\
    &= \inf_{\lambda,\eta': \lambda>0, \eta'\in\R, \frac{\sup_{x\in\cX}l(x)-\eta'}{\lambda} \leq 1} ~ \E_{P^o} [ (l(X)-\eta')_+] + \lambda\rho + \eta'.
\end{align*} 
The second equality follows since $f^*(\frac{l(X)-\eta}{\lambda})=+\infty$ whenever $\frac{l(X)-\eta}{\lambda}>\frac{1}{2}$, which can be ignored as we are minimizing over $\lambda$ and $\eta$. The fourth equality follows form the fact that $\max\{x,y\} = (x-y)_+ + y$ for any $x,y\in\R$. Finally, the last equality follows by making the substitution $\eta'=\eta-\lambda/2$. Taking the optimal value of $\lambda$, i.e.,~$\lambda=(\sup_{x\in\cX}l(x)-\eta')_+$, we get
\begin{align*}
     \sup_{D_f(P\|P^o) \leq \rho} \E_{P}[l(X)] = \inf_{\eta\in\R} \E_{P^o}[(l(X)-\eta)_+] + (\sup_{x\in\cX}l(x) - \eta)_+ \rho + \eta.  
\end{align*}
Now,
\begin{align*}
    \inf_{D_f(P\|P^o) \leq \rho} \E_{P}[l(X)]  &=  -\sup_{D_f(P\|P^o) \leq \rho} \E_{P}[-l(X)] \\
    &=  -\inf_{\eta\in\R} ~~ \E_{P^o}[(-l(X)-\eta)_+] + (\sup_{x\in\cX} -l(x) - \eta)_+ \rho + \eta \\
     &=  -\inf_{\eta' \in\R} ~~ \E_{P^o}[(\eta'-l(X))_+] + (\eta' - \inf_{x\in\cX} l(x))_+ \rho - \eta', 
\end{align*}
which completes the proof. 
\end{proof}

We are now ready to prove Proposition \ref{prop:robust-bellman-dual}. 

\begin{proof}[\textbf{Proof of Proposition~\ref{prop:robust-bellman-dual}}]
For each $(s, a)$, the optimization problem in~\eqref{eq:robust-bellman-primal} is given by $\min_{P_{s,a} \in \mathcal{P}_{s,a}} \mathbb{E}_{s' \sim P_{s,a}} [ V(s')]$, and our focus is on the setting where $\mathcal{P}_{s,a}$ is given by the TV uncertainty set. So, $\mathcal{P}_{s,a}$ can be equivalently defined using the $f$-divergence with $f(t) = |t-1|/2$ as $\mathcal{P}_{s,a} = \{P_{s,a} : D_{f}(P_{s,a}||P^{o}_{s,a}) \leq \rho\}$. We can now use the result of Lemma~\ref{lem:TV-inner-problem} to get
\begin{align*}
     \inf_{P_{s,a} \in \mathcal{P}_{s,a}} \mathbb{E}_{s' \sim P_{s,a}} [ V(s')] =  -\inf_{\eta \in\R} ~~ \E_{s' \sim P^o_{s,a}}[(\eta-V(s'))_+] + (\eta - \inf_{s'' \in \cS} V(s''))_+ \rho - \eta.
\end{align*}

From Proposition~\ref{prop:dro-inner-problem-solution}, the function $h(\eta)=\E_{s' \sim P^{o}_{s,a}}[(\eta -V(s'))_+] + \rho (\eta- \inf_{s''} V(s''))_+ - \eta  $ is convex in $\eta$.  Since $V(s') \geq 0,$ $h(\eta)=-\eta \geq 0$ when $\eta \leq 0$. So, $\inf_{\eta \in (-\infty, 0]} h(\eta)$, achieved at $\eta = 0$.  Also, since $V(s)\leq 1/(1-\gamma)$, we have
\begin{align*}
    h(\frac{2}{\rho(1-\gamma)}) &= \E_{s' \sim P^{o}_{s,a}}[\frac{2}{\rho(1-\gamma)} -V(s')]  + \rho (\frac{2}{\rho(1-\gamma)}- \inf_{s''} V(s'')) - \frac{2}{\rho(1-\gamma)}  \\
    &\geq -\frac{1}{(1-\gamma)}  + \rho (\frac{2}{\rho(1-\gamma)}- \frac{1}{(1-\gamma)} )  = \frac{2}{(1-\gamma)} - \frac{(1+\rho)}{(1-\gamma)} \geq 0. 
\end{align*} 
So, it is sufficient to consider $\eta \in [0, \frac{2}{\rho (1-\gamma)}]$ for the above optimization problem. 

Using these, we get
\begin{align*}
     (T Q)&(s, a) = r(s, a) + \gamma \inf_{P_{s,a} \in \mathcal{P}_{s,a}} \mathbb{E}_{s' \sim P_{s,a}} [ V(s')] \\
     &= r(s, a) +  \gamma \cdot  -1 \cdot \inf_{\eta \in \eta \in [0, \frac{2}{\rho (1-\gamma)}]} ~~ \E_{s' \sim P^o_{s,a}}[(\eta-V(s'))_+] + (\eta - \inf_{s'' \in \cS} V(s''))_+ \rho - \eta.
\end{align*}
This completes the proof of Proposition~\ref{prop:robust-bellman-dual}.
\end{proof}

\section{Proof of Theorem \ref{thm:tv-guarantee}}
\label{appendix:thm:tv-guarantee}

We start by proving Lemma \ref{lem:interchanging-lemma-1} which mainly follows from Lemma~\ref{lem:integral-min-exchange-rockafellar} in Appendix~\ref{appendix:useful-tech-results}.

\begin{proof}[Proof of Lemma \ref{lem:interchanging-lemma-1}]
Let $h((s,a),\eta)=\E_{s'\sim P^o_{s,a}} ((\eta - \max_{a'} f(s',a'))_+  - (1-\rho)\eta)$. We note that $h((s,a),\eta)$ is $\Sigma(\ScA)$-measurable in $(s,a)\in\ScA$ for each $\eta\in[0,{1}/{(\rho(1-\gamma))}]$ and is continuous in $\eta$ for each $(s,a)\in\ScA$. Now it follows that $h((s,a),\eta)$ is a normal integrand (see Definition~\ref{def:decomposable-normal-integrand} in Appendix~\ref{appendix:useful-tech-results}). We now note that $L^{1}(\ScA,\Sigma(\ScA),\mu)$ is a decomposable space (Remark \ref{remark:decomposable} in Appendix~\ref{appendix:useful-tech-results}). Thus, this lemma now directly follows from Lemma~\ref{lem:integral-min-exchange-rockafellar}.
\end{proof}

Now we state a result and provide its proof for the empirical risk minimization on the dual parameter.

\begin{lemma}[Dual Optimization Error Bound] \label{lem:erm-high-prob-bound}
Let $\widehat{g}_{f}$ be the dual optimization parameter from the algorithm (Step 4) for the state-action value function $f$ and let $T_{g}$ be as defined in \eqref{eq:Tg}. With probability at least $1-\delta$, we have
\[
\sup_{f\in\cF} \|T f -  T_{\widehat{g}_{f}} f\|_{1,\mu} \leq \frac{4\gamma(2-\rho)}{\rho(1-\gamma)}\sqrt{\frac{2\log(|\cG|)}{N}} +  \frac{25\gamma}{\rho(1-\gamma)} \sqrt{\frac{2\log(8|\cF|/\delta)}{N}} + \gamma \epsilon_{\textrm{dual}}. 
\]
\end{lemma}

\begin{proof}
Fix an $f\in\cF$. We will also invoke union bound for the supremum here. We recall from~\eqref{eq:L-dual-empirical-loss} that $\widehat{g}_f = \argmin_{g \in \cG} \widehat{L}_{\mathrm{dual}}(g; f)$. From  the robust Bellman equation, we directly obtain
\begingroup
\allowdisplaybreaks
\begin{align*}
    &\|T_{\widehat{g}_{f}} f - T f \|_{1,\mu} = \gamma (\E_{s,a\sim\mu} | \E_{s'\sim P^o_{s,a}} ((\widehat{g}_{f}(s,a) - \max_{a'} f(s',a'))_+  - (1-\rho)\widehat{g}_{f}(s,a)) \\&\hspace{2cm}- \inf_{\eta \in[0, {2}/{(\rho(1-\gamma))}]}  \E_{s'\sim P^o_{s,a}} ((\eta  - \max_{a'} f(s',a'))_+  - (1-\rho) \eta) | ) \\
    &\stackrel{(a)}{=} \gamma (\E_{s,a\sim\mu}\E_{s'\sim P^o_{s,a}} ((\widehat{g}_{f}(s,a) - \max_{a'} f(s',a'))_+  - (1-\rho)\widehat{g}_{f}(s,a)) \\&\hspace{1cm}- \E_{s,a\sim\mu} [ \inf_{\eta \in[0, {2}/{(\rho(1-\gamma))}]}  \E_{s'\sim P^o_{s,a}} ((\eta  - \max_{a'} f(s',a'))_+  - (1-\rho) \eta) ]  ) \\
    &\stackrel{(b)}{=} \gamma (\E_{s,a\sim\mu,s'\sim P^o_{s,a}} ((\widehat{g}_{f}(s,a) - \max_{a'} f(s',a'))_+  - (1-\rho)\widehat{g}_{f}(s,a)) \\&\hspace{1cm}- \inf_{g \in L^{1}} \E_{s,a\sim\mu,s'\sim P^o_{s,a}} ((g(s,a) - \max_{a'} f(s',a'))_+  - (1-\rho)g(s,a))  )\\
    &= \gamma (\E_{s,a\sim\mu,s'\sim P^o_{s,a}} ((\widehat{g}_{f}(s,a) - \max_{a'} f(s',a'))_+  - (1-\rho)\widehat{g}_{f}(s,a)) \\&\hspace{1cm}- \inf_{g\in \cG} \E_{s,a\sim\mu,s'\sim P^o_{s,a}} ((g(s,a) - \max_{a'} f(s',a'))_+  - (1-\rho)g(s,a))  )
    \\ &\hspace{1cm} + \gamma (\inf_{g\in \cG} \E_{s,a\sim\mu,s'\sim P^o_{s,a}} ((g(s,a) - \max_{a'} f(s',a'))_+  - (1-\rho)g(s,a)) 
    \\&\hspace{1.5cm}- \inf_{g \in L^{1}} \E_{s,a\sim\mu,s'\sim P^o_{s,a}} ((g(s,a) - \max_{a'} f(s',a'))_+  - (1-\rho)g(s,a))  )\\
    &\stackrel{(c)}{\leq} \gamma (\E_{s,a\sim\mu,s'\sim P^o_{s,a}} ((\widehat{g}_{f}(s,a) - \max_{a'} f(s',a'))_+  - (1-\rho)\widehat{g}_{f}(s,a)) \\&\hspace{1cm}- \inf_{g\in \cG} \E_{s,a\sim\mu,s'\sim P^o_{s,a}} ((g(s,a) - \max_{a'} f(s',a'))_+  - (1-\rho)g(s,a))  )  + \gamma \epsilon_{\textrm{dual}} \\
    &\stackrel{(d)}{\leq} 2\gamma R(l\circ \cG \circ \cD) +  \frac{25\gamma}{\rho(1-\gamma)} \sqrt{\frac{2\log(8/\delta)}{N}}   + \gamma \epsilon_{\textrm{dual}}
    \\&\stackrel{(e)}{\leq} \frac{4\gamma(2-\rho)}{\rho(1-\gamma)}\sqrt{\frac{2\log(|\cG|)}{N}} +  \frac{25\gamma}{\rho(1-\gamma)} \sqrt{\frac{2\log(8/\delta)}{N}}  + \gamma \epsilon_{\textrm{dual}}.
\end{align*}
\endgroup
$(a)$ follows since $\inf_{g} h(g) \leq h(\widehat{g}_{f})$. $(b)$ follows from Lemma \ref{lem:interchanging-lemma-1}. $(c)$ follows from the approximate dual realizability assumption (Assumption \ref{assum-dual-realizability}). 

For $(d)$, we consider the loss function $l(g,(s,a,s'))=(g(s,a) - \max_{a'} f(s',a'))_+  - (1-\rho)g(s,a)$ and dataset $\cD=\{s_i,a_i,s_i'\}_{i=1}^N$. Note that $|l(g,(s,a,s'))|\leq 5/(\rho(1-\gamma))$ (since $f\in\cF$ and $g\in\cG$). Now, we can apply the empirical risk minimization result \eqref{eq:erm-generalization-bound-part-1} in  Lemma \ref{lem:erm-generalization-bound} to get $(d)$, where $R(\cdot)$ is the Rademacher complexity. 

Finally, $(e)$ follows from~\eqref{eq:erm-generalization-bound-part-2} in Lemma \ref{lem:erm-generalization-bound} when combined with the   facts that $l(g,(s,a,s'))$ is $(2-\rho)$-Lipschitz in $g$ and $g(s,a) \leq 2/(\rho(1-\gamma))$, since $g\in\cG$.

With union bound, with probability at least $1-\delta$, we finally get \begin{align*}
    &\sup_{f\in\cF}\|T f -  T_{\widehat{g}_{f}} f\|_{1,\mu} {\leq} \frac{4\gamma(2-\rho)}{\rho(1-\gamma)}\sqrt{\frac{2\log(|\cG|)}{N}} +  \frac{25\gamma}{\rho(1-\gamma)} \sqrt{\frac{2\log(8|\cF|/\delta)}{N}} + \gamma \epsilon_{\textrm{dual}},
\end{align*}
which concludes the proof.
\end{proof}


We next prove the least-squares generalization bound for the RFQI algorithm. 

\begin{lemma}[Least squares generalization bound] \label{lem:least-squares-generalization-bound}
Let  $\widehat{f}_{g}$ be the least-squares solution from the algorithm (Step 5) for the state-action value function $f$ and  dual variable function  $g$. Let $T_{g}$ be as defined in \eqref{eq:Tg}. Then, with probability at least $1-\delta$, we have 
\begin{align*}
    \sup_{f\in\cF} \sup_{g\in\cG} \|  T_{g} f - \widehat{f}_{g} \|_{2,\mu}
    &\leq \sqrt{6 \epsilon_{\textrm{c}}} + \frac{16}{\rho(1-\gamma)} \sqrt{\frac{18 \log(2|\cF||\cG|/\delta)}{N}}.
\end{align*}
\end{lemma}

\begin{proof}
We adapt the least-squares generalization bound given in~\citet[Lemma A.11]{agarwal2019reinforcement} to our setting. We recall from~\eqref{eq:loss-rfqi-1} that $\widehat{f}_{g}=\argmin_{Q\in\cF}\widehat{L}_{\mathrm{RFQI}}(Q;f, g)$. We first fix functions $f\in\cF$ and $g\in\cG$. For any function $f'\in\cF$, we define random variables $z_i^{f'}$ as
\begin{align*}
    z_i^{f'} = \left( f'(s_i,a_i) - y_{i} \right)^2  - \left( (T_{g} f)(s_i,a_i) - y_{i}  \right)^2,
\end{align*} 
where $y_{i} =  r_i - \gamma(g(s_i,a_i) - \max_{a'} f(s_i',a'))_+  + \gamma (1-\rho)g(s_i,a_i)$, and $(s_{i}, a_{i}, s'_{i}) \in \cD$ with $(s_{i}, a_{i}) \sim \mu, s'_{i} \sim P^{o}_{s_{i}, a_{i}}$. It is straightforward to note that for a given $(s_{i}, a_{i})$, we have $ \E_{s'_i \sim P^o_{s_i,a_i}} [y_{i}] = (T_{g} f)(s_{i}, a_{i})$.

Also, since $g(s_i,a_i)\leq 2/(\rho(1-\gamma))$ (because $g\in\cG$) and $f(s_i,a_i), f'(s_i,a_i)\leq 1/(1-\gamma)$ (because $f, f'\in\cF)$, we have $(T_{g}f)(s_i,a_i) \leq 5/(\rho(1-\gamma))$. This also gives us that $y_{i} \leq 5/(\rho(1-\gamma))$. 

Using this, we obtain the first moment and an upper-bound for the second moment of $z_i^{f'}$ as follows:
\begin{align*}
    \E_{s'_i \sim P^o_{s_i,a_i}} [z_i^{f'}] &= \E_{s'_i \sim P^o_{s_i,a_i}} [(   f'(s_i,a_i) - (T_{g} f)(s_i,a_i)) \cdot (f'(s_i,a_i) + (T_{g} f)(s_i,a_i)  -2 y_{i} ) ] \\
    &= (   f'(s_i,a_i) - (T_{g} f)(s_i,a_i))^{2}, \\
    \E_{s'_i \sim P^o_{s_i,a_i}}  [(z_i^{f'})^2] &= \E_{s'_i \sim P^o_{s_i,a_i}} [(   f'(s_i,a_i) - (T_{g} f)(s_i,a_i))^{2} \cdot (f'(s_i,a_i) + (T_{g} f)(s_i,a_i)  -2 y_{i} )^{2} ] \\
     &= ( f'(s_i,a_i) - (T_{g} f)(s_i,a_i))^{2} \cdot \E_{s'_i \sim P^o_{s_i,a_i}} [  (f'(s_i,a_i) + (T_{g} f)(s_i,a_i)  -2 y_{i} )^{2} ] \\
     &\leq C_{1} ( f'(s_i,a_i) - (T_{g} f)(s_i,a_i))^{2} ,
\end{align*} 
where $C_{1}={16^2}/{(\rho^2(1-\gamma)^2)}$. This immediately implies that
\begin{align*}
     \E_{s_i,a_i\sim \mu, s'_i \sim P^o_{s_i,a_i}} [z_i^{f'}] &= \norm{T_{g}f -f'}^{2}_{2,\mu}, \\ 
     \E_{s_i,a_i\sim \mu, s'_i \sim P^o_{s_i,a_i}} [(z_i^{f'})^2] &\leq C_{1} \norm{T_{g}f -f'}^{2}_{2,\mu}.
\end{align*}
From these calculations, it is also straightforward to see that $|z_i^{f'}-\E_{s_i,a_i\sim \mu, s'_i \sim P^o_{s_i,a_i}} [z_i^{f'}]| \leq 2 C_{1}$ almost surely.

Now, using  the Bernstein's inequality (Lemma \ref{lem:bernstein-ineq}), together with a union bound over all $f'\in \cF$, with probability at least $1-\delta$, we have
\begin{align} 
\label{eq:least-squares-bern-main}
    |\|  T_{g} f - f' \|_{2,\mu}^2 - \frac{1}{N} \sum_{i=1}^N z_i^{f'}| \leq \sqrt{\frac{2 C_{1}     \|  T_{g} f - f' \|_{2,\mu}^2 \log(2|\cF|/\delta)}{N}} + \frac{2C_{1} \log(2|\cF|/\delta)}{3N},
\end{align} 
for all $f'\in \cF$. Setting $f'=\widehat{f}_{g}$, with probability at least $1-\delta/2$, we have 
\begin{align} 
\label{eq:least-squares-bern-1}
    \|  T_{g} f - \widehat{f}_{g} \|_{2,\mu}^2  \leq \frac{1}{N} \sum_{i=1}^N z_i^{\widehat{f}_{g}}+ \sqrt{\frac{2 C_{1}     \|  T_{g} f - \widehat{f}_{g} \|_{2,\mu}^2 \log(4|\cF|/\delta)}{N}} + \frac{2C_{1} \log(4|\cF|/\delta)}{3N}.
\end{align}
Now we upper-bound $(1/N) \sum_{i=1}^N z_i^{\widehat{f}_{g}}$ in the following. Consider a function $\widetilde{f} \in \argmin_{h\in\cF} \| h -  T_{g} f \|_{2,\mu}^2$. Note that $\widetilde{f}$ is independent of the dataset. We note that our earlier first and second moment calculations hold true for $\widetilde{f}$, replacing $f'$, as well. Now, from \eqref{eq:least-squares-bern-main} setting $f'=\widetilde{f}$, with probability at least $1-\delta/2$ we have
\begin{align} \label{eq:least-squares-z-main}
    \frac{1}{N} \sum_{i=1}^N z_i^{\widetilde{f}} -\|  T_{g} f - \widetilde{f} \|_{2,\mu}^2   \leq \sqrt{\frac{2 C_{1}     \|  T_{g} f - \widetilde{f} \|_{2,\mu}^2 \log(4|\cF|/\delta)}{N}} + \frac{2C_{1} \log(4|\cF|/\delta)}{3N}.
\end{align}
Suppose $(1/N) \sum_{i=1}^N z_i^{\widetilde{f}} \geq 2C_{1}\log(4|\cF|/\delta)/N$ holds, then from \eqref{eq:least-squares-z-main} we get 
\begin{align} \label{eq:least-squares-z-main-2}
    \frac{1}{N} \sum_{i=1}^N z_i^{\widetilde{f}} -\|  T_{g} f - \widetilde{f} \|_{2,\mu}^2   \leq \sqrt{     \|  T_{g} f - \widetilde{f} \|_{2,\mu}^2 \cdot \frac{1}{N} \sum_{i=1}^N z_i^{\widetilde{f}}} + \frac{2C_{1} \log(4|\cF|/\delta)}{N}.
\end{align}
We note the following algebra fact: Suppose $x^2 - ax + b \leq 0$ with $b>0$ and $a^2\geq 4b$, then we have $x\leq a$. Taking $x=(1/N) \sum_{i=1}^N z_i^{\widetilde{f}}$ in this fact, from~\eqref{eq:least-squares-z-main-2} we get 
\begin{align} \label{eq:least-squares-z-part-1}
    \frac{1}{N} \sum_{i=1}^N z_i^{\widetilde{f}}    \leq 3\|  T_{g} f - \widetilde{f} \|_{2,\mu}^2 + \frac{4C_{1} \log(4|\cF|/\delta)}{3N} \leq 3\|  T_{g} f - \widetilde{f} \|_{2,\mu}^2 + \frac{2C_{1} \log(4|\cF|/\delta)}{N}.
\end{align}
Now suppose $(1/N) \sum_{i=1}^N z_i^{\widetilde{f}} \leq 2 C_{1}\log(4|\cF|/\delta)/N$, then \eqref{eq:least-squares-z-part-1} holds immediately. Thus,~\eqref{eq:least-squares-z-part-1} always holds with probability at least $1-\delta/2$. Furthermore, recall $\widetilde{f} \in \argmin_{h\in\cF} \| h -  T_{g} f \|_{2,\mu}^2$, we have 
\begin{align}
    \frac{1}{N} \sum_{i=1}^N z_i^{\widetilde{f}} &\leq 3\|  T_{g} f - \widetilde{f} \|_{2,\mu}^2 + \frac{2C_{1} \log(4|\cF|/\delta)}{N} \nonumber \\
    &= 3\min_{h\in\cF} \| h -  T_{g} f \|_{2,\mu}^2 + \frac{2C_{1} \log(4|\cF|/\delta)}{N} \leq 3 \epsilon_{\textrm{c}} + \frac{2C_{1} \log(4|\cF|/\delta)}{N}
     \label{eq:least-squares-z-part-2},
\end{align}
where the last inequality follows from the approximate robust Bellman completion assumption (Assumption \ref{assum-bellman-completion}). 

We note that since $\widehat{f}_{g}$ is the least-squares regression solution, we know that $(1/N) \sum_{i=1}^N z_i^{\widehat{f}_{g}} \leq (1/N) \sum_{i=1}^N z_i^{\widetilde{f}}$.  With this note in \eqref{eq:least-squares-z-part-2}, from \eqref{eq:least-squares-bern-1}, with probability at least $1-\delta$, we have 
\begin{align*}
    \|  T_{g} f - \widehat{f}_{g} \|_{2,\mu}^2  &\leq  3 \epsilon_{\textrm{c}} + \frac{2C_{1} \log(4|\cF|/\delta)}{N} 
    \\&+ \sqrt{\frac{2 C_{1}     \|  T_{g} f - \widehat{f}_{g} \|_{2,\mu}^2 \log(4|\cF|/\delta)}{N}} + \frac{2C_{1} \log(4|\cF|/\delta)}{3N} \\
    &\leq 3 \epsilon_{\textrm{c}} + \frac{3C_{1} \log(4|\cF|/\delta)}{N} + \sqrt{\frac{3 C_{1}     \|  T_{g} f - \widehat{f}_{g} \|_{2,\mu}^2 \log(4|\cF|/\delta)}{N}}.
\end{align*}
From the earlier algebra fact, taking $x=\|  T_{g} f - \widehat{f}_{g} \|_{2,\mu}^2$, with probability at least $1-\delta$, we have
\begin{align*}
    \|  T_{g} f - \widehat{f}_{g} \|_{2,\mu}^2
    &\leq 6 \epsilon_{\textrm{c}} + \frac{9C_{1} \log(4|\cF|/\delta)}{N}.
\end{align*}
From the fact $\sqrt{x+y}\leq \sqrt{x}+\sqrt{y}$, with probability at least $1-\delta$, we get
\begin{align*}
    \|  T_{g} f - \widehat{f}_{g} \|_{2,\mu}
    &\leq \sqrt{6 \epsilon_{\textrm{c}}} + \sqrt{\frac{9C_{1} \log(4|\cF|/\delta)}{N}}.
\end{align*}
Using union bound for $f\in\cF$ and $g\in\cG$, with probability at least $1-\delta$, we finally obtain 
\begin{align*}
    \sup_{f\in\cF} \sup_{g\in\cG} \|  T_{g} f - \widehat{f}_{g} \|_{2,\mu}
    &\leq \sqrt{6 \epsilon_{\textrm{c}}} + \sqrt{\frac{18 C_{1} \log(2|\cF||\cG|/\delta)}{N}},
\end{align*}
which completes the least-squares generalization bound analysis.
\end{proof}



We are now ready to prove the main theorem.

\begin{proof}[\textbf{Proof of Theorem \ref{thm:tv-guarantee}}]
We let $V_k(s) = Q_k(s,\pi_k(s))$ for every $s\in\cS$. Since $\pi_k$ is the greedy  policy w.r.t $Q_k$, we also have $V_k(s) = Q_k(s,\pi_k(s)) = \max_a Q_k(s,a)$. We recall that $V^*=V^{\pi^*}$ and $Q^*=Q^{\pi^*}$. We also recall from Section~\ref{sec:preliminaries} that $Q^{\pi^*}$ is a fixed-point of the robust Bellman operator $T$ defined in~\eqref{eq:robust-bellman-primal}. We also note that the same holds true for any stationary deterministic policy $\pi$ from~\citet{iyengar2005robust} that $Q^{\pi}$ satisfies $Q^{\pi}(s, a) = r(s, a) + \gamma \min_{P_{s,a} \in \mathcal{P}_{s,a}} \mathbb{E}_{s' \sim P_{s,a}} [ V^{\pi}(s')].$ We can now further use the dual form \eqref{eq:robust-bellman-dual-2} under Assumption \ref{assum-fail-state}. We first characterize the performance decomposition between $V^{\pi^*}$ and ${V}^{\pi_K}$. For a given $s_0\in\cS$, we observe that
\begin{align*}
	&V^{\pi^*}(s_0) - {V}^{\pi_K}(s_0) = (V^{\pi^*}(s_0) -  V_K(s_0) )  - (V^{\pi_K}(s_0) -  V_K(s_0) )\\ 
	&=  (Q^{\pi^*}(s_0,\pi^*(s_0)) -  Q_K(s_0,\pi_K(s_0)) ) - (Q^{\pi_K}(s_0,\pi_K(s_0)) -  Q_K(s_0,\pi_K(s_0))) \\ 
	&\stackrel{(a)}{\leq} Q^{\pi^*}(s_0,\pi^*(s_0)) -  Q_K(s_0,\pi^*(s_0)) +  Q_K(s_0,\pi_K(s_0)) - Q^{\pi_K}(s_0,\pi_K(s_0))  \\
	&= Q^{\pi^*}(s_0,\pi^*(s_0)) -  Q_K(s_0,\pi^*(s_0)) +  Q_K(s_0,\pi_K(s_0)) - Q^{\pi^*}(s_0,\pi_K(s_0))
	\\&\hspace{5cm}+ Q^{\pi^*}(s_0,\pi_K(s_0)) - Q^{\pi_K}(s_0,\pi_K(s_0))
	\\
	&\stackrel{(b)}{\leq} Q^{\pi^*}(s_0,\pi^*(s_0)) -  Q_K(s_0,\pi^*(s_0)) +  Q_K(s_0,\pi_K(s_0)) - Q^{\pi^*}(s_0,\pi_K(s_0)) \\&\hspace{2cm}+ \gamma \sup_{\eta} (\E_{s_1\sim P^o_{s_0,\pi_{K}(s_0)}}(  (\eta- V^{\pi_K}(s_1))_+ - (\eta-V^{\pi^*}(s_1))_+))  \\
	&\stackrel{(c)}{\leq} |Q^{\pi^*}(s_0,\pi^*(s_0)) -  Q_K(s_0,\pi^*(s_0))| + |Q^{\pi^*}(s_0,\pi_K(s_0)) -  Q_K(s_0,\pi_K(s_0))|  
	\\&\hspace{4cm}+ \gamma \E_{s_1\sim P^o_{s_0,\pi_{K}(s_0)}}(  |V^{\pi^*}(s_1)- V^{\pi_K}(s_1)| ).
\end{align*}
$(a)$ follows from the fact that $\pi_{K}$ is the greedy policy with respect to $Q_{K}$. $(b)$ follows from the Bellman optimality equations and the fact $|\sup_{x} f(x)-\sup_{x} g(x)| \leq \sup_{x} |f(x) - g(x)|$. Finally, $(c)$ follows from the facts $(x)_+ - (y)_+ \leq (x-y)_+$ and $(x)_+ \leq |x|$ for any $x,y\in\R$.

We now recall the initial state distribution $d_0$. Thus, we have 
\begin{align*}  &\E_{s_0\sim d_0}[{V}^{\pi^*}] - \E_{s_0\sim d_0}[V^{\pi_K}] \leq
\\&\hspace{0.5cm} \E_{s_0\sim d_0} \bigg[ |Q^{\pi^*}(s_0,\pi^*(s_0)) -  Q_{K}(s_0,\pi^*(s_0))| + |Q^{\pi^*}(s_0,\pi_K(s_0)) -  Q_{K}(s_0,\pi_K(s_0))| 
	\\&\hspace{3cm}+ \gamma \E_{s_1\sim P^{o}_{s_0,\pi_{K}(s_0)}} (|V^{\pi^*}(s_1) - {V}^{\pi_K}(s_1) |) \bigg]. 
\end{align*}
Since $V^{\pi^*}(s) \geq V^{\pi_K}(s)$ for any $s\in\cS$, by telescoping we get 
\begin{align} 
&\E_{s_0\sim d_0}[{V}^{\pi^*}] - \E_{s_0\sim d_0}[V^{\pi_K}] \leq \sum_{h=0}^\infty \gamma^h \times \nonumber \\
&\hspace{1cm} \bigg(\E_{s\sim d_{h,\pi_{K}}} [ |Q^{\pi^*}(s,\pi^*(s)) -  Q_{K}(s,\pi^*(s))| + |Q^{\pi^*}(s,\pi_K(s)) -  Q_{K}(s,\pi_K(s))| ] \bigg),  \label{eq:perf-diff-1} 
\end{align} 
where $d_{h,\pi_K}\in \Delta(\cS)$ for all natural numbers $h\geq 0$ is defined as
\begin{equation*}
  d_{h,\pi_{K}} =
    \begin{cases}
      d_0 & \text{if $h=0$},\\
      P^{o}_{s',\pi_{K}(s')} & \text{otherwise, with } s'\sim d_{h-1,\pi_{K}}.
    \end{cases}       
\end{equation*}
We emphasize that the state distribution $d_{h,\pi_K}$'s are different from the discounted state-action occupancy distributions. We note that a similar state distribution proof idea is used in \citet{agarwal2019reinforcement}.

Recall $\|f\|_{p,\nu}^2 = (\E_{s,a\sim\nu} |f(s,a)|^p)^{1/p}$, where $\nu\in\Delta(\ScA)$. With this we have 
\begin{align}  
\label{eq:thm-bound-part-1}
&\E_{s_0\sim d_0}[{V}^{\pi^*}] - \E_{s_0\sim d_0}[V^{\pi_K}] \leq \sum_{h=0}^\infty \gamma^h \bigg(    \|Q^{\pi^*} -  Q_{K}\|_{1,d_{h,\pi_{K}}\circ \pi^* } + \|Q^{\pi^*} -  Q_{K}\|_{1,d_{h,\pi_{K}}\circ \pi_K } \bigg), 
\end{align} 
where the state-action distributions $d_{h,\pi_{K}}\circ \pi^*(s,a)\propto d_{h,\pi_{K}}(s) \mathbbm{1}\{ a=\pi^*(s)\}$ and $d_{h,\pi_{K}}\circ \pi_K(s,a)\propto d_{h,\pi_{K}}(s) \mathbbm{1}\{ a=\pi_K(s)\}$ directly follows by comparing with \eqref{eq:perf-diff-1}.

We now bound one of the RHS terms above by bounding for any state-action distribution $\nu$ satisfying Assumption \ref{assum-concentra-condition} (in particular the following bound is true for $d_{h,\pi_{K}}\circ \pi^*$ or $d_{h,\pi_{K}}\circ \pi_K$ in \eqref{eq:perf-diff-1}): 
\begin{align}
    &\|Q^{\pi^*} -  Q_{K}\|_{1,\nu} \leq \|Q^{\pi^*} -  T Q_{K-1}\|_{1,\nu} + \|T Q_{K-1} -  Q_{K}\|_{1,\nu} \nonumber \\
    &\stackrel{(a)}{\leq} \|Q^{\pi^*} -  T Q_{K-1}\|_{1,\nu} + \sqrt{C}\|T Q_{K-1} -  Q_{K}\|_{1,\mu} \nonumber \\
    &= (\E_{s,a\sim\nu} |Q^{\pi^*}(s,a) -  T Q_{K-1}(s,a)|) + \sqrt{C}\|T Q_{K-1} -  Q_{K}\|_{1,\mu} \nonumber \\
    &\stackrel{(b)}{\leq} (\E_{s,a\sim\nu}   \gamma \sup_{\eta}|
    \E_{s'\sim P^o_{s,a}} (  (\eta-\max_{a'} Q_{K-1}(s',a'))_+ - (\eta-\max_{a'} Q^{\pi^*}(s',a'))_+) 
    |)
    \nonumber \\&\hspace{6cm}+ \sqrt{C}\|T Q_{K-1} -  Q_{K}\|_{1,\mu} \nonumber \\
    &\stackrel{(c)}{\leq} (\E_{s,a\sim\nu}  | 
    \E_{s'\sim P^o_{s,a}}   (\max_{a'} Q^{\pi^*}(s',a')-\max_{a'} Q_{K-1}(s',a'))_+|) + \sqrt{C}\|T Q_{K-1} -  Q_{K}\|_{1,\mu} \nonumber \\
    &\stackrel{(d)}{\leq} \gamma (\E_{s,a\sim\nu} \E_{s'\sim {P}^o_{s,a}} \max_{a'} |Q^{\pi^*}(s',a') -  Q_{K-1}(s',a')| )     + \sqrt{C}\|T Q_{K-1} -  Q_{K}\|_{1,\mu} \nonumber\\
    &\stackrel{(e)}{\leq} \gamma \|Q^{\pi^*} -  Q_{K-1}\|_{1,\nu'} + \sqrt{C}\|T Q_{K-1} -  Q_{K}\|_{1,\mu} \nonumber  \\
    \label{eq:pf-sketech-stp4-1}
    &\stackrel{(f)}{\leq} \gamma \|Q^{\pi^*} -  Q_{K-1}\|_{1,\nu'} + \sqrt{C}\|T_{g_{K-1}} Q_{K-1} -  Q_{K}\|_{2,\mu}+ \sqrt{C}\|T Q_{K-1} -  T_{g_{K-1}} Q_{K-1}\|_{1,\mu},
\end{align} 
where $(a)$ follows by the concentratability assumption (Assumption \ref{assum-concentra-condition}), $(b)$ from Bellman equation, operator $T$, and the fact $|\sup_{x} p(x)-\sup_{x} q(x)| \leq \sup_{x} |p(x) - q(x)|$,
$(c)$ from the fact $|(x)_+ - (y)_+| \leq |(x-y)_+|$ for any $x,y\in\R$,
$(d)$ follows by Jensen's inequality and by the facts $|\sup_{x} p(x)-\sup_{x} q(x)| \leq \sup_{x} |p(x) - q(x)|$ and $(x)_+ \leq |x|$ for any $x,y\in\R$, and $(e)$ by defining the distribution $\nu'$ as $\nu'(s',a')=\sum_{s,a} \nu(s,a) {P}^o_{s,a}(s') \mathbbm{1} \{ a'= \argmax_b |Q^{\pi^*}(s',b)- Q_{K-1}(s',b)| \}$, and $(f)$ using the fact that $\|\cdot\|_{1,\mu}\leq \|\cdot\|_{2,\mu}$.

Now, by recursion until iteration 0, we get \begin{align} \|&Q^{\pi^*} -  Q_{K}\|_{1,\nu} \leq \gamma^K \sup_{\bar{\nu}} \|Q^{\pi^*} -  Q_{0}\|_{1,\bar{\nu}} + \sqrt{C} \sum_{t=0}^{K-1} \gamma^t \|T Q_{K-1-t} -  T_{g_{K-1-t}} Q_{K-1-t}\|_{1,\mu} \nonumber
\\&\hspace{4cm}+ \sqrt{C} \sum_{t=0}^{K-1} \gamma^t  \| T_{g_{K-1-t}} Q_{K-1-t} -  Q_{K-t}\|_{2,\mu}
\nonumber\\
&\stackrel{(a)}{\leq} \frac{\gamma^K }{1-\gamma} + \sqrt{C} \sum_{t=0}^{K-1} \gamma^t \|T Q_{K-1-t} -  T_{g_{K-1-t}} Q_{K-1-t}\|_{1,\mu}
\nonumber\\&\hspace{4cm}+ \sqrt{C} \sum_{t=0}^{K-1} \gamma^t  \| T_{g_{K-1-t}} Q_{K-1-t} -  Q_{K-t}\|_{2,\mu}
\nonumber\\&\stackrel{(b)}{\leq} \frac{\gamma^K }{1-\gamma} + \frac{\sqrt{C} }{1-\gamma} \sup_{f\in\cF} \|T f -  T_{\widehat{g}_{f}} f\|_{1,\mu}
+ \frac{\sqrt{C} }{1-\gamma} \sup_{f\in\cF} \| T_{\widehat{g}_{f}} f -  \widehat{f}_{\widehat{g}_{f}}\|_{2,\mu}
\nonumber\\&\leq \frac{\gamma^K }{1-\gamma} + \frac{\sqrt{C} }{1-\gamma} \sup_{f\in\cF} \|T f -  T_{\widehat{g}_{f}} f\|_{1,\mu}
+ \frac{\sqrt{C} }{1-\gamma} \sup_{f\in\cF} \sup_{g\in\cG} \| T_{g} f -  \widehat{f}_{g}\|_{2,\mu}.  \label{eq:thm-bound-part-2} \end{align} where $(a)$ follows since $|Q^{\pi^*}(s,a)|\leq 1/(1-\gamma), Q_{0}(s,a)= 0$, and $(b)$ follows since $\widehat{g}_{f}$ is the dual variable function from the algorithm for the state-action value function $f$ and $\widehat{f}_{g}$ as the least squares solution from the algorithm for the state-action value function $f$ and dual variable function $g$ pair.

The proof is now complete combining \eqref{eq:thm-bound-part-1} and \eqref{eq:thm-bound-part-2} with Lemma \ref{lem:erm-high-prob-bound} and Lemma \ref{lem:least-squares-generalization-bound}.
\end{proof}

\section{Related Works}
\label{sec:app:related-works}

Here we provide a more detailed description of the related work to complement what we listed in the introduction (Section~\ref{sec:introduction}).

\textbf{Offline RL:} The problem of learning the optimal policy only using an offline dataset is first addressed under the generative model assumption \citep{singh1994upper, AzarMK13, haskell2016empirical, sidford2018near, agarwal2020model,li2020breaking, kalathil2021empirical}. This assumption requires generating the same uniform number of next-state samples for each and every state-action pairs.
To account for large state spaces, there are number of works \citep{antos2008learning, bertsekas2011approximate,  lange2012batch, chen2019information, xie2020q, levine2020offline, xie2021bellman} that utilize function approximation under similar assumption, concentratability assumption \citep{chen2019information} in which the data distribution $\mu$ sufficiently covers the discounted state-action occupancy. There is rich literature \citep{munos08a, farahmand2010error, lazaric2012finite, chen2019information, liu2020provably, xie2021bellman} in the conquest of identifying and improving these necessary and sufficient assumptions for offline RL that use variations of Fitted Q-Iteration (FQI) algorithm \citep{gordon1995stable, ernst2005tree}.
There is also rich literature \citep{fujimoto2019off, kumar2019stabilizing, kumar2020conservative, yu2020mopo, zhang2021towards} that develop offline deep RL algorithms focusing on the algorithmic and empirical aspects and propose multitude heuristic approaches to advance the field. All these results assume that the offline data is generated according to a single model and the goal is to find the optimal policy for the MDP with the same model. In particular, none of these works consider the \textit{offline robust RL problem} where the offline data is generated according to a (training) model which can be different from the one in testing, and the goal is to learn a policy that is robust w.r.t.~an uncertainty set.

\textbf{Robust RL:} To address the parameter uncertainty problem,~\citet{iyengar2005robust} and~\citet{nilim2005robust} introduced the RMDP framework.~\citet{iyengar2005robust} showed that the optimal robust value function and policy can be computed using the robust counterparts of the standard value iteration and policy iteration algorithms. To tackle the parameter uncertainty problem, other works considered distributionally robust setting \citep{xu2010distributionally}, modified policy iteration \citep{kaufman2013robust}, and {more general uncertainty set} \citep{wiesemann2013robust}. These initial works mainly focused on the planning problem (known transition probability dynamics) in the tabular setting.~\citet{tamar2014scaling} proposed linear function approximation method to solve large RMDPs. Though this work suggests a sampling based approach, a general model-free learning algorithm and analysis was not included.~\citet{roy2017reinforcement} proposed the robust versions of the classical model-free reinforcement learning algorithms, such as Q-learning, SARSA, and TD-learning in the tabular setting. They also proposed function approximation based algorithms for  the policy evaluation. However,  this work does not have a policy iteration algorithm with provable guarantees for learning the optimal robust policy.~\citet{derman2018soft} introduced soft-robust actor-critic algorithms using neural networks, but does not provide any global convergence guarantees for the learned policy.~\citet{tessler2019action} proposed a min-max game framework to address the robust learning problem focusing on the tabular setting.~\citet{lim2019kernel} proposed a kernel-based RL algorithm for finding the robust value function in a batch learning setting.~\citet{Mankowitz2020Robust} employed an entropy-regularized policy optimization algorithm for continuous control using neural network, but does not provide any provable guarantees for the learned policy.~\citet{panaganti2020robust} proposed least-squares policy iteration method to handle large state-action space in robust RL, but only provide asymptotic policy evaluation convergence guarantees whereas \citet{panaganti2020robust} provide finite time convergence for the policy iteration to optimal robust value.

\textbf{Other robust RL related works:} Robust control is a well-studied area in the classical control theory~\citep{zhou1996robust, dullerud2013course}. Recently, there are some interesting works that address the robust RL  problem using this framework, especially focusing on the linear quadratic regulator setting \citep{zhang2020policy}.  Risk sensitive RL algorithms \citep{borkar2002q,prashanth2016variance,fei2021exponential} and adversarial RL algorithms \citep{pinto2017robust,zhang2021provably,huang2022robust}  also address the robustness problem implicitly under different frameworks which are independent from RMDPs. \citep{zhang2022corruption} addresses the problem of \textit{corruption-robust} offline RL problem, where an adversary is allowed to change a fraction of the samples of an offline RL dataset and the goal is to find the optimal policy for the nominal linear MDP model (according to which the offline data is generated).
Our framework  and approach of robust MDP is significantly different from these line of works.

\textbf{This work:} The works that are closest to ours are by \citet{zhou2021finite,yang2021towards,panaganti22a} that address the robust RL problem in a tabular setting under the generative model assumption. Due to the generative model assumption, the offline data has the same uniform number of samples corresponding to each and every state-action pair, and tabular setting allows the estimation of the uncertainty set followed by solving the planning problem. Our work is significantly different from these in the following way: $(i)$ we consider a robust RL problem with arbitrary large state space, instead of the small tabular setting, $(ii)$ we consider a true offline RL setting where the state-action pairs are sampled according to an arbitrary distribution, instead of using the generative model assumption, $(iii)$ we focus on a function approximation approach where the goal is to directly learn optimal robust value/policy using function approximation techniques, instead of solving the  tabular planning problem with the estimated model. \textit{To the best of our knowledge, this is the first work that addresses the offline robust RL problem with arbitrary large state space  using function approximation, with provable guarantees on the performance of the learned policy.}

\section{Experiment Details}
\label{appendix:simulations}

We provide more detailed and practical version of our RFQI algorithm (Algorithm \ref{alg:RFQI-Algorithm}) in this section. We also provide more experimental results evaluated on \textit{Cartpole}, \textit{Hopper}, and \textit{Half-Cheetah} OpenAI Gym Mujoco \citep{brockman2016openai} environments.

We provide our code in \textbf{github webpage} \url{https://github.com/zaiyan-x/RFQI} containing instructions to reproduce all results in this paper.
We implemented our RFQI algorithm based on the architecture of Batch Constrained deep Q-learning (BCQ) algorithm \citep{fujimoto2019off} \footnote{Available at \url{https://github.com/sfujim/BCQ}}  and  Pessimistic Q-learning (PQL) algorithm \citep{liu2020provably} \footnote{Available at \url{https://github.com/yaoliucs/PQL}}. We note that PQL algorithm (with $b=0$ filtration thresholding \citep{liu2020provably}) and BCQ algorithm  are the practical versions of FQI algorithm with neural network architecture. 

\subsection{RFQI Practical Algorithm} \label{appendix-sub:rfqi-prac-implementation}
We provide the practical version of our RFQI algorithm in Algorithm \ref{alg:prac-RFQI-Algorithm} and highlight the difference with BCQ and PQL algorithms in \bluetext{blue} (steps 8 and 9).

\begin{algorithm}[t]
	\caption{RFQI Practical Algorithm}	
	\label{alg:prac-RFQI-Algorithm}
	\begin{algorithmic}[1]
		\STATE \textbf{Input:} Offline dataset $\cD$, radius of robustness $\rho$, maximum perturbation $\Phi$, target update rate $\tau$, mini-batch size $N$, maximum number of iterations $K$, number of actions $u$. 
		\STATE \textbf{Initialize:} Two state-action neural networks $Q_{\theta_1}$ and $Q_{\theta_2}$, \bluetext{one dual neural network $g_{\theta_3}$} \\policy (perturbation) model: $\xi_{\phi} \in [-\Phi, \Phi]),$ and action VAE $G^a_{\omega}$,  with random parameters $\theta_1$, $\theta_2$, $\phi$, $\omega$, and target networks $Q_{\theta'_1}, Q_{\theta'_2}$, $\xi_{\phi'}$ with $\theta'_1 \leftarrow \theta_1, \theta'_2 \leftarrow \theta_2$, $\phi' \leftarrow \phi$.
		\FOR {$k=1,\cdots,K$ } 
		\STATE Sample a minibatch $B$ with $N$ samples from $\cD$. 
		\STATE Train $\omega \leftarrow\argmin_{\omega} ELBO(B;G^a_{\omega}).$ Sample $u$ actions $a_i'$ from $G^a_{\omega} (s')$ for each $s'$. 
		\STATE Perturb $u$ actions $a_i'=a_i' + \xi_{\phi}(s',a_i')$.
		\STATE Compute next-state value target for each $s'$ in $B$: $$V_t = \max_{a'_i} (0.75\cdot \min\{ Q_{\theta'_1}, Q_{\theta'_2} \} + 0.25\cdot \max\{ Q_{\theta'_1}, Q_{\theta'_2} \}).$$
		\STATE \vspace{-0.3cm} \bluetext{$\theta_3 \leftarrow \argmin_{\theta} \sum [\max\{g_{\theta}(s,a) - V_t(s'), 0\}  - (1-\rho) g_{\theta}(s,a)].$}
		\STATE \bluetext{Compute next-state Q target for each $(s, a, r, s' )$ pair in $B$:  $$ Q_t(s,a) = r - \gamma \cdot \max\{g_{\theta_3}(s,a) - V_t(s'), 0\}  + \gamma (1-\rho) g_{\theta_3}(s,a). $$}
        \STATE \vspace{-0.3cm} $\theta \leftarrow \argmin_{\theta} \sum (Q_t(s,a)  - Q_{\theta}(s,a))^2.$
        \STATE Sample $u$ actions $a_i$ from $G^a_{\omega} (s)$ for each $s$. 
        \STATE $\phi \leftarrow \argmax_{\phi} \sum \max_{a_i} Q_{\theta_1}(s,a_i+\xi_{\phi}(s,a_i))$.
        \STATE Update target network: $\theta'=(1-\tau)\theta' + \tau\theta, \phi'=(1-\tau)\phi' + \tau\phi$.
		\ENDFOR
		\STATE \textbf{Output policy:} Given $s$, sample $u$ actions $a_i$ from $G^a_{\omega} (s)$. Select action $a  = \argmax_{a_i} Q_{\theta_1}(s,a_i+\xi_{\phi}(s,a_i))$.
	\end{algorithmic}
\end{algorithm}

\textbf{RFQI algorithm implementation details}: The Variational Auto-Encoder (VAE) $G^a_\omega$ \citep{kingma2013auto} is defined by two networks, an encoder $E_{\omega_1}(s,a)$ and decoder $D_{\omega_2}(s, z)$, where $\omega = \{\omega_1, \omega_2\}$. The encoder outputs mean and standard deviation, $(\mu, \sigma) = E_{\omega_1}(s,a)$, of a normal distribution. 
A latent vector $z$ is sampled from the standard normal distribution and for a state $s$, the decoder maps them to an action $D_{\omega_2}:(s, z)\mapsto\tilde{a}$. Then the evidence lower bound ($ELBO$) of VAE is given by $ELBO(B;G^a_{\omega})=\sum_{B} (a - \tilde{a})^2 + D_{\text{KL}}(\cN(\mu, \sigma),\cN(0,1)),$ where $\cN$ is the normal distribution with mean and standard deviation parameters. We refer to \citep{fujimoto2019off} for more details on VAE. We also use the default VAE architecture from BCQ algorithm \citep{fujimoto2019off} and PQL algorithm \citep{liu2020provably} in our RFQI algorithm.

We now focus on the additions described in blue (steps 8 and 9) in Algorithm \ref{alg:prac-RFQI-Algorithm}. For all the other networks we use default architecture from BCQ algorithm \citep{fujimoto2019off} and PQL algorithm \citep{liu2020provably} in our RFQI algorithm.\\
(1) In each iteration $k$, we solve the dual variable function $g_\theta$ optimization problem (step 4 in Algorithm \ref{alg:RFQI-Algorithm}, step 8 in Algorithm \ref{alg:prac-RFQI-Algorithm})  implemented by ADAM \citep{kingma2014adam} on the minibatch $B$ with the learning rate $l_1$ mentioned in Table \ref{tab:hp}.\\
(2) Our state-action value target function corresponds to the robust state-action value target function described in \eqref{eq:loss-rfqi-1}. This is reflected in step 9 of Algorithm \ref{alg:prac-RFQI-Algorithm}. The state-action value function $Q_\theta$ optimization problem (step 5 in Algorithm \ref{alg:RFQI-Algorithm}, step 9 in Algorithm \ref{alg:prac-RFQI-Algorithm}) is  implemented by ADAM \citep{kingma2014adam} on the minibatch $B$ with the learning rate $l_2$ mentioned in Table \ref{tab:hp}.

\begin{table}[H]
	\begin{center}
		\begin{tabular}{|c|c|c|c|c|}
			\hline 
			Environment & Discount & Learning rates  & Q Neural nets & Dual Neural nets \\ 
			 & $\gamma$ & $[l_1,l_2]$  & $\theta_1=\theta_2=[h_1,h_2]$ & $\theta_3=[h_1,h_2]$ \\ 
			\hline & &&& \vspace{-0.3cm}\\
			CartPole & $0.99$ & $[10^{-3},10^{-3}]$  & $[400, 300]$ & $[64, 64]$\\
			Hopper & $0.99$ & $[10^{-3},8\times 10^{-4}]$  & $[400, 300]$ & $[64, 64]$\\
			 & & $[3\times 10^{-4},6\times10^{-4}]$  & & \\
			Half-Cheetah & $0.99$ & $[10^{-3},8\times 10^{-4}]$  & $[400, 300]$ & $[64, 64]$\\
			 & & $[3\times 10^{-4},6\times10^{-4}]$  & & \\
			\hline
		\end{tabular}
	\end{center}\caption{Details of hyper-parameters in FQI and RFQI algorithms experiments.} \label{tab:hp}
\end{table}
\textbf{Hyper-parameters details}: We now give the description of hyper-parameters used in our codebase  in Table \ref{tab:hp}. We use same hyper-parameters across different algorithms. Across all learning algorithms we use $\tau=0.005$ for the target network update, $K=5\times10^5$ for the maximum iterations, $|\cD|=10^6$ for the offline dataset, $|B|=1000$ for the minibatch size. We used grid-search for $\rho$ in $\{0.2,0.3,\cdots,0.6\}$. We also picked best of the two sets of learning rates mentioned in Table \ref{tab:hp}. For all the other hyper-parameters we use default values from BCQ algorithm \citep{fujimoto2019off} and PQL algorithm \citep{liu2020provably} in our RFQI algorithm that can be found in our code.

\textbf{Offline datasets}: Now we discuss the offline dataset used in the our training of FQI and RFQI algorithms. For the fair comparison in every plot, we train both FQI and RFQI algorithms on same offline datasets.

\textit{Cartpole dataset} $\cD_{\textrm{c}}$: We first train proximal policy optimization (PPO) \citep{schulman2017proximal} algorithm, under default RL baseline zoo \citep{rl-zoo3} parameters. We then generate the Cartpole dataset $\cD_{\textrm{c}}$ with $10^5$ samples  using  an $\epsilon$-greedy ($\epsilon=0.3$) version of this PPO trained policy. We note that this offline dataset contains non-expert behavior meeting the richness of the data-generating distribution assumption in practice. 

\textit{Mixed dataset} $\cD_{\textrm{m}}$: For the MuJoCo environments, \textit{Hopper} and \textit{Half-Cheetah},  we increase the richness of the dataset since these are high dimensional problems. We first train soft actor-critic (SAC) \citep{haarnoja2018soft} algorithm, under default RL baseline zoo \citep{rl-zoo3} parameters,   with replay buffer updated by a fixed $\epsilon$-greedy ($\epsilon=0.1$) policy with the model parameter \textit{actuator\_ctrlrange} set to $[-0.85,0.85]$. We then generate  the mixed dataset $\cD_{\textrm{m}}$ with $10^6$ samples from this $\epsilon$-greedy ($\epsilon=0.3$) SAC trained policy. We note that such a dataset generation gives more diverse set of observations than the process of $\cD_{\textrm{c}}$ generation for fair comparison between FQI and RFQI algorithms.

\textit{D4RL dataset} $\cD_{\textrm{d}}$: We consider the \textit{hopper-medium} and \textit{halfcheetah-medium} offline datasets in \citep{fu2020d4rl} which are benchmark datasets in offline RL literature \citep{fu2020d4rl,levine2020offline,liu2020provably}. These `medium' datasets are generated by first training a policy online using Soft Actor-Critic \citep{haarnoja2018soft}, early-stopping the training, and collecting $10^6$ samples from this partially-trained policy. We refer to \citep{fu2020d4rl} for more details.

We end this  section by mentioning the software and hardware configurations used. The training and evaluation is done using  three computers with the following configuration. Operating system is Ubuntu 18.04 and Lambda Stack; main softwares are PyTorch, Caffe, CUDA, cuDNN, Numpy, Matplotlib; processor is AMD Threadripper 3960X (24 Cores, 3.80 GHz); GPUs are 2x RTX 2080 Ti; memory is 128GB RAM; Operating System Drive is 1 TB SSD (NVMe); and Data Drive is 4TB HDD.

\subsection{More Experimental Results}

\begin{figure*}[ht]
\vspace{-0.5cm}
	\centering
	\begin{minipage}{.32\textwidth}
		\centering
		\includegraphics[width=\linewidth]{figures/CartPolePerturbed-v0_action.pdf}
	\end{minipage}
	\begin{minipage}{.32\textwidth}
		\centering
		\includegraphics[width=\linewidth]{figures/CartPolePerturbed-v0_force_mag.pdf}
	\end{minipage}
	\begin{minipage}{.32\textwidth}
		\centering
		\includegraphics[width=\linewidth]{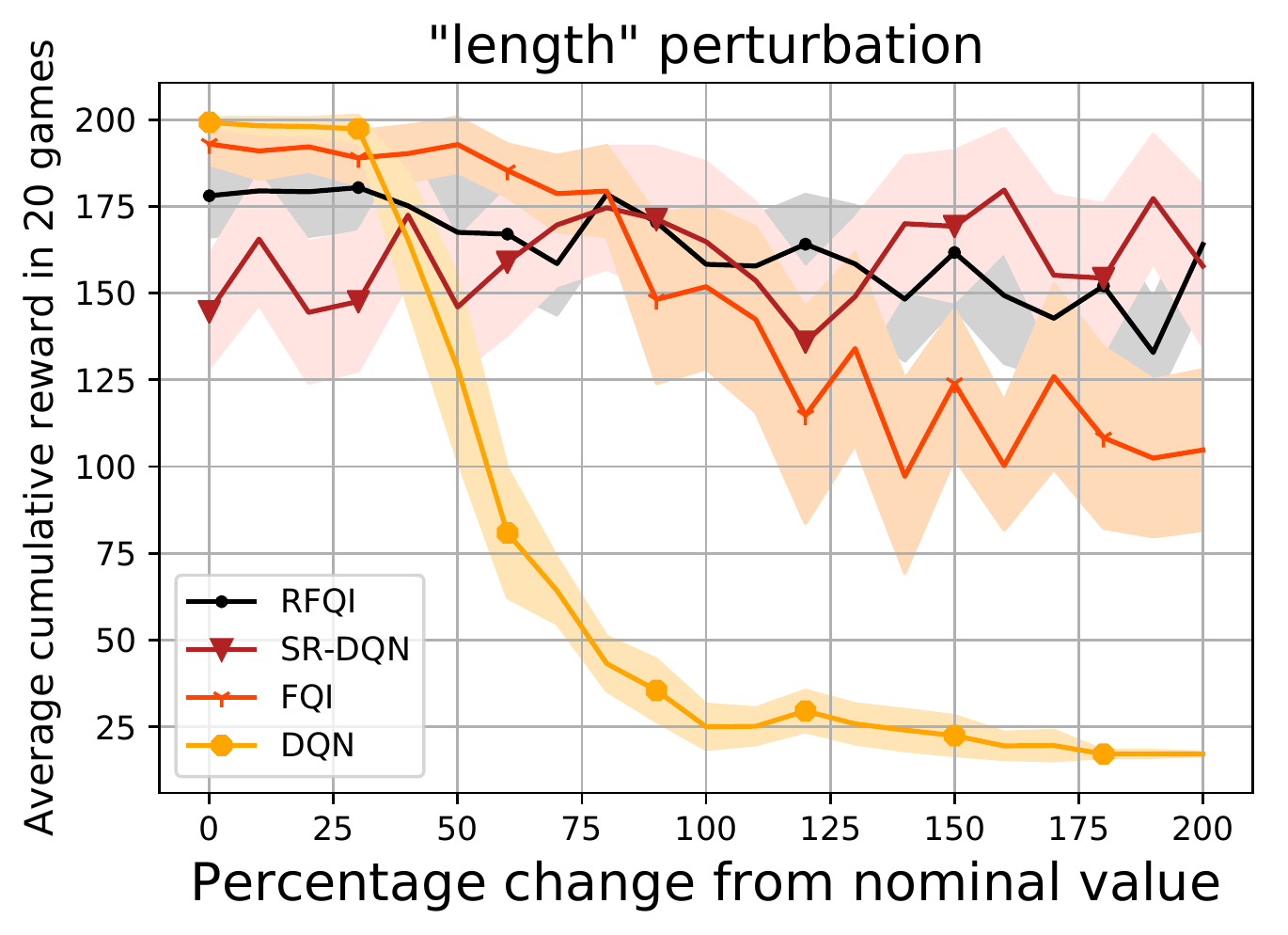}
	\end{minipage}
	\captionof{figure}{ \textit{Cartpole  simulation results on offline dataset $\cD_{\textrm{c}}$.} Average cumulative reward in $20$ episodes versus different model parameter perturbations mentioned in the respective titles.}
	\label{fig:cartpole}
	\centering
	\begin{minipage}{.32\textwidth}
		\centering
		\includegraphics[width=\linewidth]{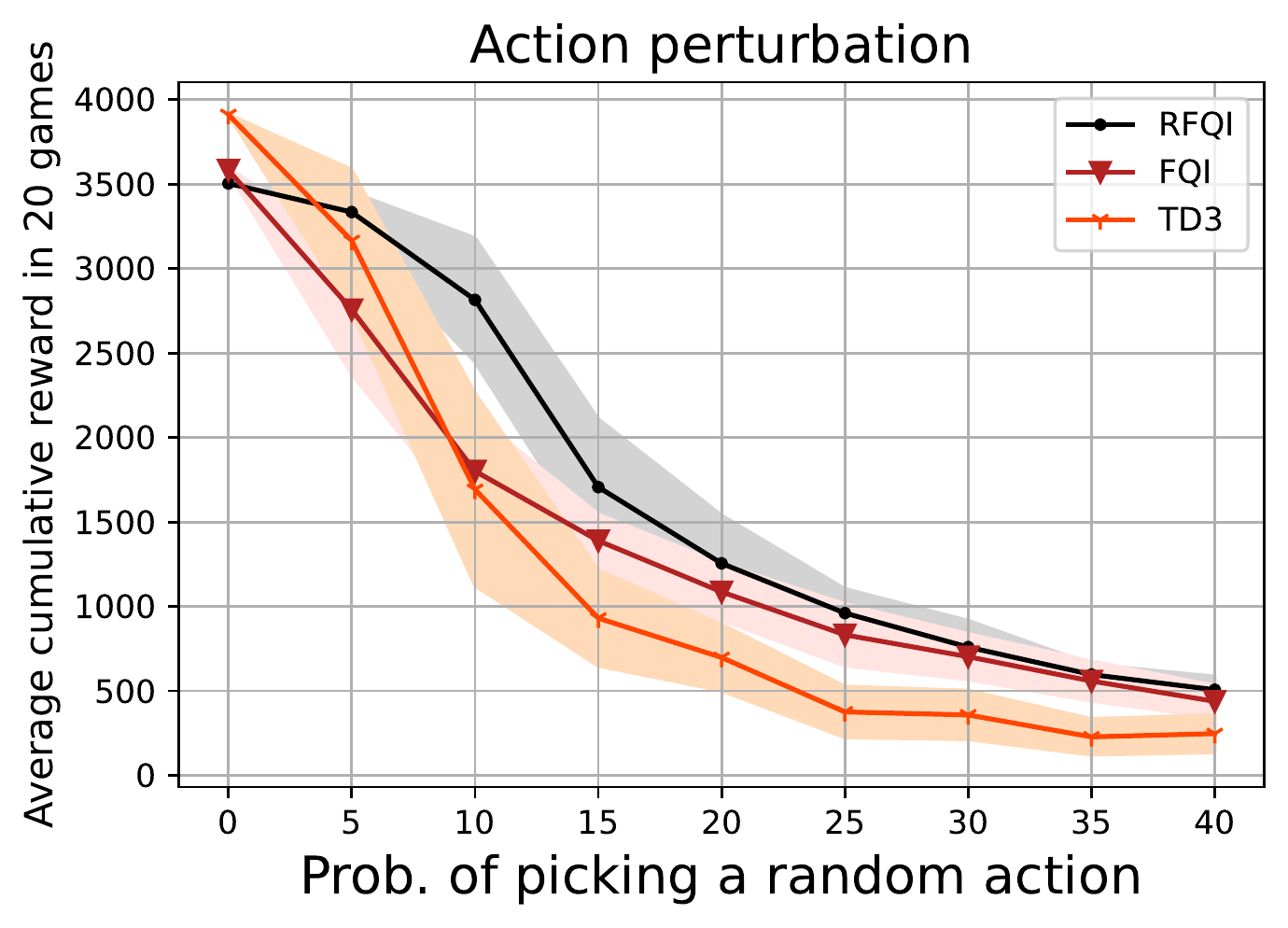}
	\end{minipage}
	\begin{minipage}{.32\textwidth}
		\centering
		\includegraphics[width=\linewidth]{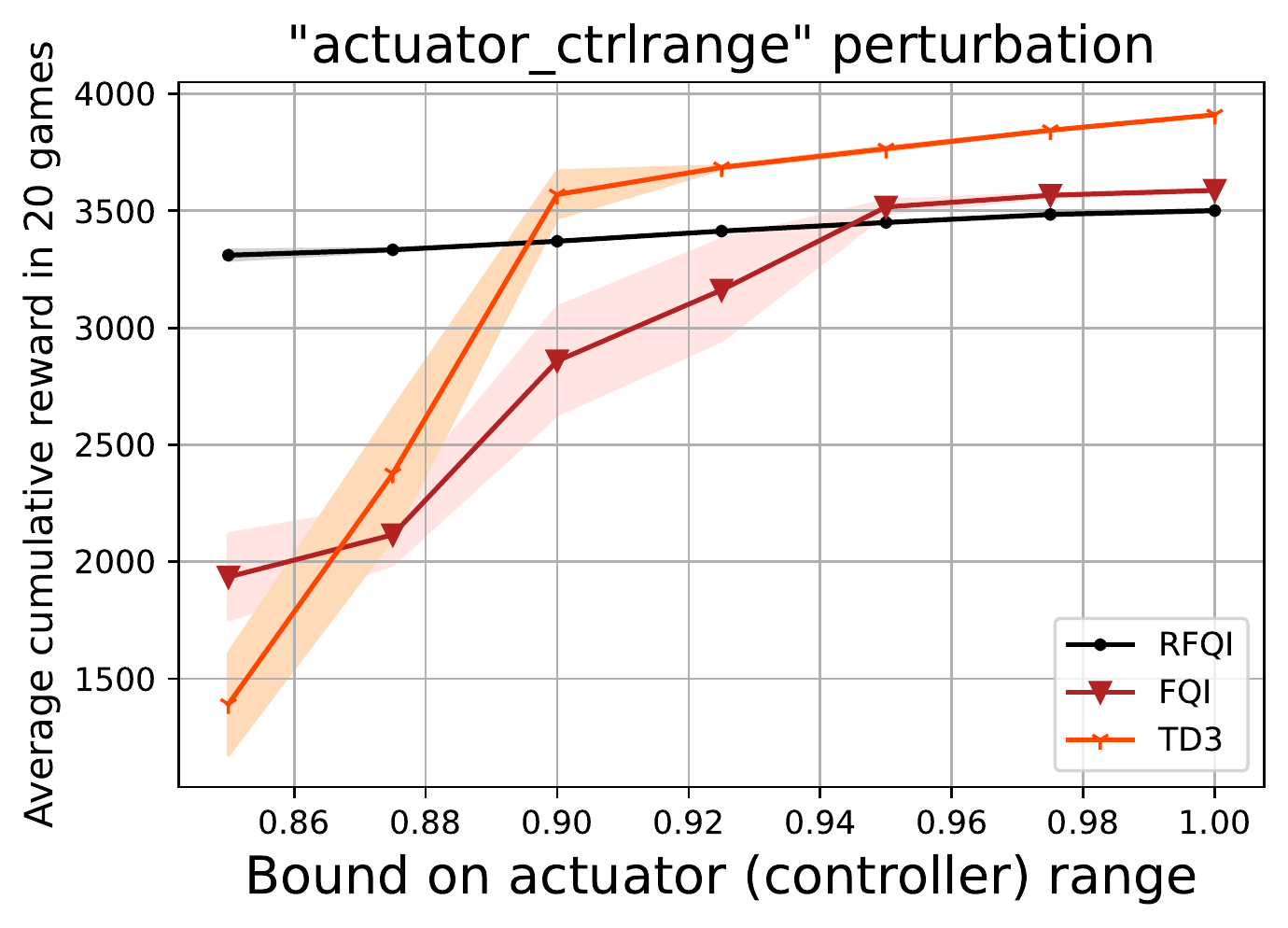}
	\end{minipage}
	\begin{minipage}{.32\textwidth}
		\centering
		\includegraphics[width=\linewidth]{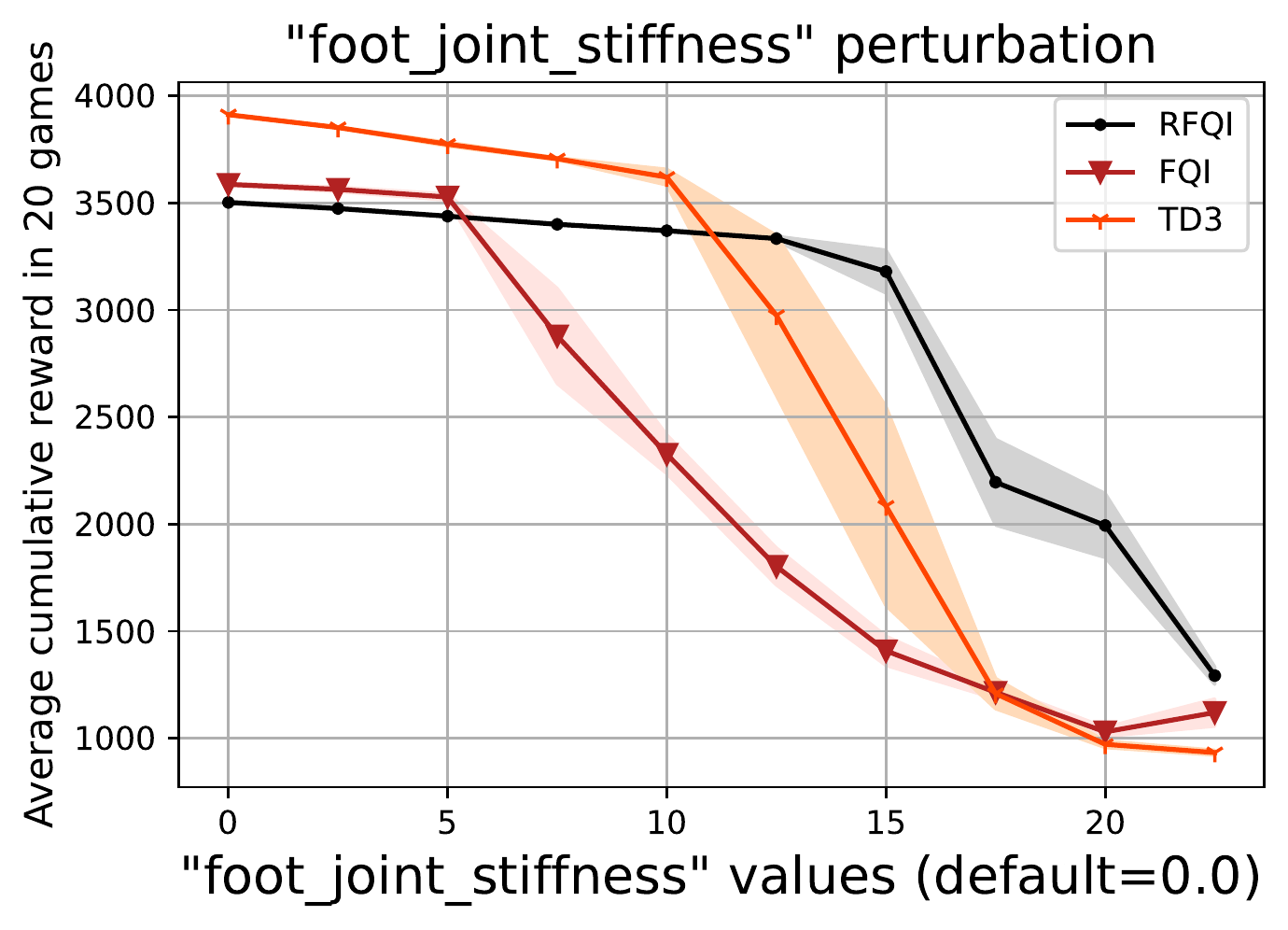}
	\end{minipage}
	\captionof{figure}{ \textit{Hopper  simulation results on offline dataset $\cD_{\textrm{m}}$.} Average cumulative reward in $20$ episodes versus different model parameter perturbations mentioned in the respective titles.}
	\label{fig:hopper-mixed}
\end{figure*}

Here we provide more experimental results and details in addition to Fig. \ref{fig:cart_force-in}-\ref{fig:hopper_leg} in Section \ref{sec:main-experiments}.

For the \textit{Cartpole}, we compare RFQI algorithm against the non-robust RL algorithms FQI and DQN, and the soft-robust RL algorithm proposed in \cite{derman2018soft}. We trained FQI and RFQI algorithms  on the dataset $\cD_{\textrm{c}}$ (a detailed description of data set is provided in Appendix \ref{appendix-sub:rfqi-prac-implementation}).
We test the robustness of the algorithms by changing the parameters \textit{force\_mag} (to model external force disturbance), \textit{length} (to model change in pole length), and also by introducing action perturbations (to model actuator noise). The nominal value of \textit{force\_mag} and \textit{length} parameters are $10$ and $0.5$ respectively. Fig. \ref{fig:cartpole} shows superior robust performance of RFQI compared to the non-robust FQI and DQN.
For example, consider the action perturbation performance plot in Fig. \ref{fig:cartpole} where RFQI algorithm improves by $75\%$ compared to FQI algorithm in average cumulative reward for a $40\%$ chance of action perturbation.
We note that we found $\rho=0.5$ is the best from grid-search for RFQI algorithm. The RFQI performance is similar to that of  soft-robust DQN. We note that soft-robust DQN algorithm is an online deep RL algorithm (and not an offline RL algorithm) and has no provable performance guarantee. Moreover,  soft-robust DQN algorithm requires generating  online data according a number of models in the uncertainty set, whereas RFQI only requires offline  data according to a single nominal training model.

Before we proceed to describe our results on the OpenAI Gym MuJoCo \citep{brockman2016openai} environments \textit{Hopper} and \textit{Half-Cheetah}, we first mention their model parameters and its corresponding nominal values in Table \ref{tab:model-parameters}. The model parameter names are self-explanatory, for example, stiffness control on the leg joint is the \textit{leg\_joint\_stiffness}, range of actuator values is  the \textit{actuator\_ctrlrange}. The front and back parameters in Half-Cheetah are for the front and back legs. We refer to the perturbed environments provided in our code and the \textit{hopper.xml, halfcheetah.xml} files in the environment assets of OpenAI Gym MuJoCo \citep{brockman2016openai} for more information regarding these model parameters.

\begin{table}[H]
	\begin{center}
		\begin{tabular}{|c|c|c|}
			\hline 
			Environment & Model parameter & Nominal range/value  \\  
			\hline & & \vspace{-0.3cm}\\
			Hopper & \textit{actuator\_ctrlrange} & $[-1,1]$ \\
			 & \textit{foot\_joint\_stiffness} & $0$ \\
			 & \textit{leg\_joint\_stiffness} & $0$ \\
			 & \textit{thigh\_joint\_stiffness} & $0$ \\
			 & \textit{joint\_damping} & $1$ \\
			 & \textit{joint\_frictionloss} & $0$ \\
			\hline & & \vspace{-0.3cm}\\
			 & \textit{joint\_frictionloss} & $0$ \\
			Half-Cheetah & front \textit{actuator\_ctrlrange} & $[-1,1]$ \\
			& back \textit{actuator\_ctrlrange} & $[-1,1]$ \\
			& front \textit{joint\_stiffness} = (\textit{thigh\_joint\_stiffness}, & \\
			& \textit{shin\_joint\_stiffness}, \textit{foot\_joint\_stiffness}) & $(180,120,60)$ \\
			& back \textit{joint\_stiffness} = (\textit{thigh\_joint\_stiffness}, & \\
			& \textit{shin\_joint\_stiffness}, \textit{foot\_joint\_stiffness}) & $(240,180,120)$ \\
			& front \textit{joint\_damping} = (\textit{thigh\_joint\_damping}, & \\
			& \textit{shin\_joint\_damping}, \textit{foot\_joint\_damping}) & $(4.5,3.0,1.5)$ \\
			& back \textit{joint\_damping} = (\textit{thigh\_joint\_damping}, & \\
			& \textit{shin\_joint\_damping}, \textit{foot\_joint\_damping}) & $(6.0,4.5,3.0)$ \\
			\hline
		\end{tabular}
	\end{center}\caption{Details of model parameters for \textit{Hopper} and \textit{Half-Cheetah} environments.} \label{tab:model-parameters}
\end{table}
\vspace{-0.5cm}

For the \textit{Hopper}, we compare RFQI algorithm against the non-robust RL algorithms FQI and TD3 \citep{fujimoto2018addressing}. We trained FQI and RFQI algorithms on the mixed dataset $\cD_{\textrm{m}}$  (a detailed description of dataset provided in Appendix \ref{appendix-sub:rfqi-prac-implementation}). We note that we do not compare with soft robust RL algorithms because of its poor performance on MuJoCo environments in the rest of our figures. We test the robustness of the algorithm by introducing action perturbations, and by changing the model parameters  \textit{actuator\_ctrlrange}, \textit{foot\_joint\_stiffness}, and \textit{leg\_joint\_stiffness}. 
Fig. \ref{fig:hopper_leg} and Fig. \ref{fig:hopper-mixed}  shows RFQI algorithm is consistently robust compared to the non-robust algorithms. We note that we found $\rho=0.5$ is the best from grid-search for RFQI algorithm. The  average episodic reward of RFQI remains almost the same initially, and later decays much less and gracefully   when compared to FQI and TD3 algorithms. For example, in plot 3 in Fig. \ref{fig:hopper-mixed}, at the \textit{foot\_joint\_stiffness} parameter value $15$, the episodic reward of FQI is only around $1400$ whereas RFQI achieves an episodic reward of $3200$. Similar robust performance of RFQI can be seen in other plots as well. We also note that TD3 \citep{fujimoto2019off} is a powerful off-policy policy gradient algorithm that relies on large $10^6$ replay buffer of online data collection, unsurprisingly performs well initially with less perturbation near the nominal models.


In order to verify the effectiveness and consistency of our algorithm across different offline dataset, we repeat the above experiments, on additional OpenAI Gym MuJoCo \citep{brockman2016openai} environment \textit{Half-Cheetah}, using  D4RL dataset $\cD_{\mathrm{d}}$ (a detailed description of dataset provided in Appendix \ref{appendix-sub:rfqi-prac-implementation}) which are benchmark in offline RL literature \citep{fu2020d4rl,levine2020offline,liu2020provably} than our mixed dataset $\cD_{\mathrm{m}}$. Since  D4RL dataset is a benchmark dataset for offline RL algorithms, here we focus only on the comparison between the two offline RL algorithms we consider, our  RFQI algorithm and its non-robust counterpart FQI algorithm. We now showcase the results on \textit{Hopper} and \textit{Half-Cheetah} for this setting.

\begin{figure*}[t]
	\centering
	\begin{minipage}{.32\textwidth}
		\centering
		\includegraphics[width=\linewidth]{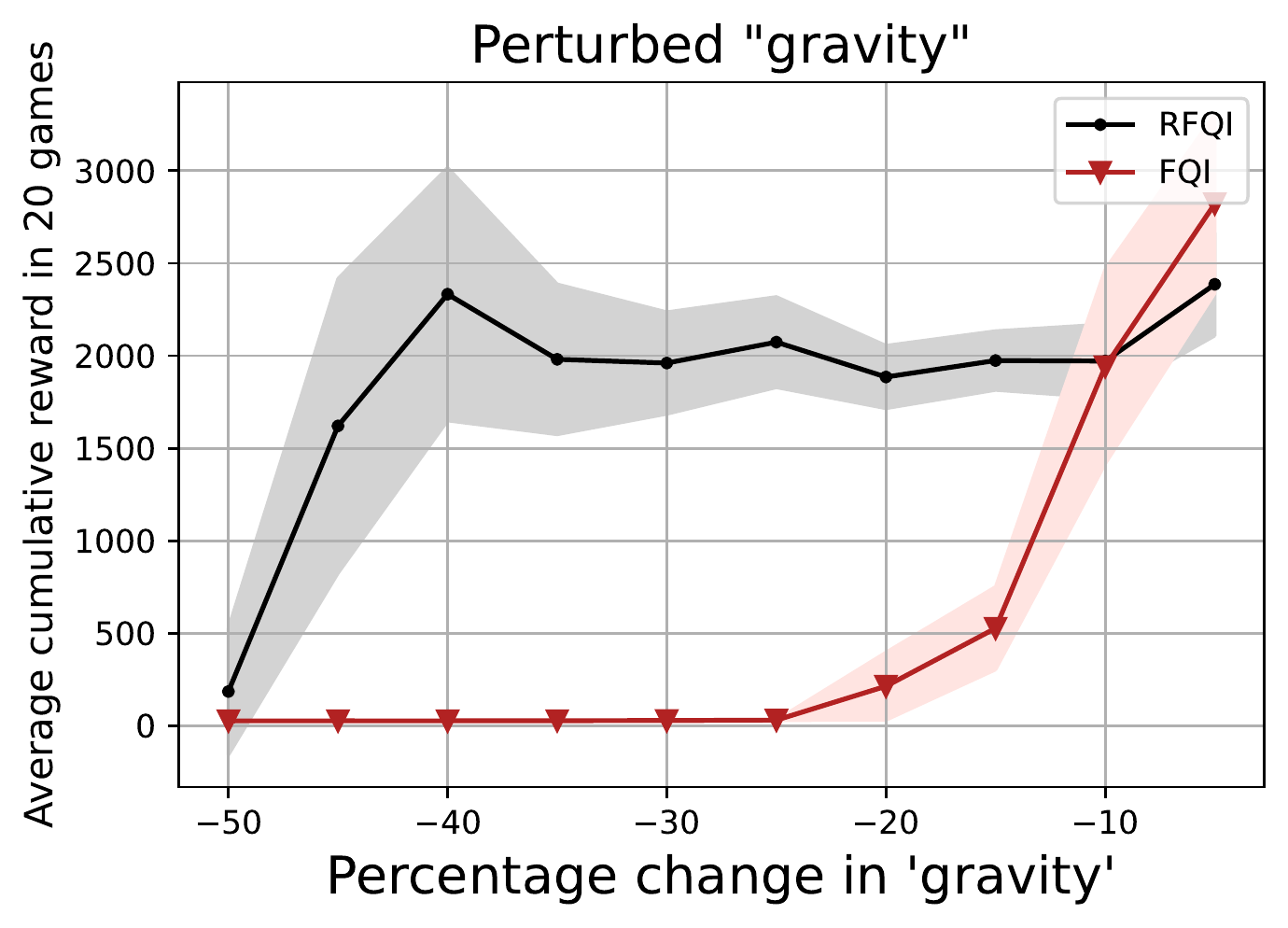}
	\end{minipage}
	\begin{minipage}{.32\textwidth}
		\centering
		\includegraphics[width=\linewidth]{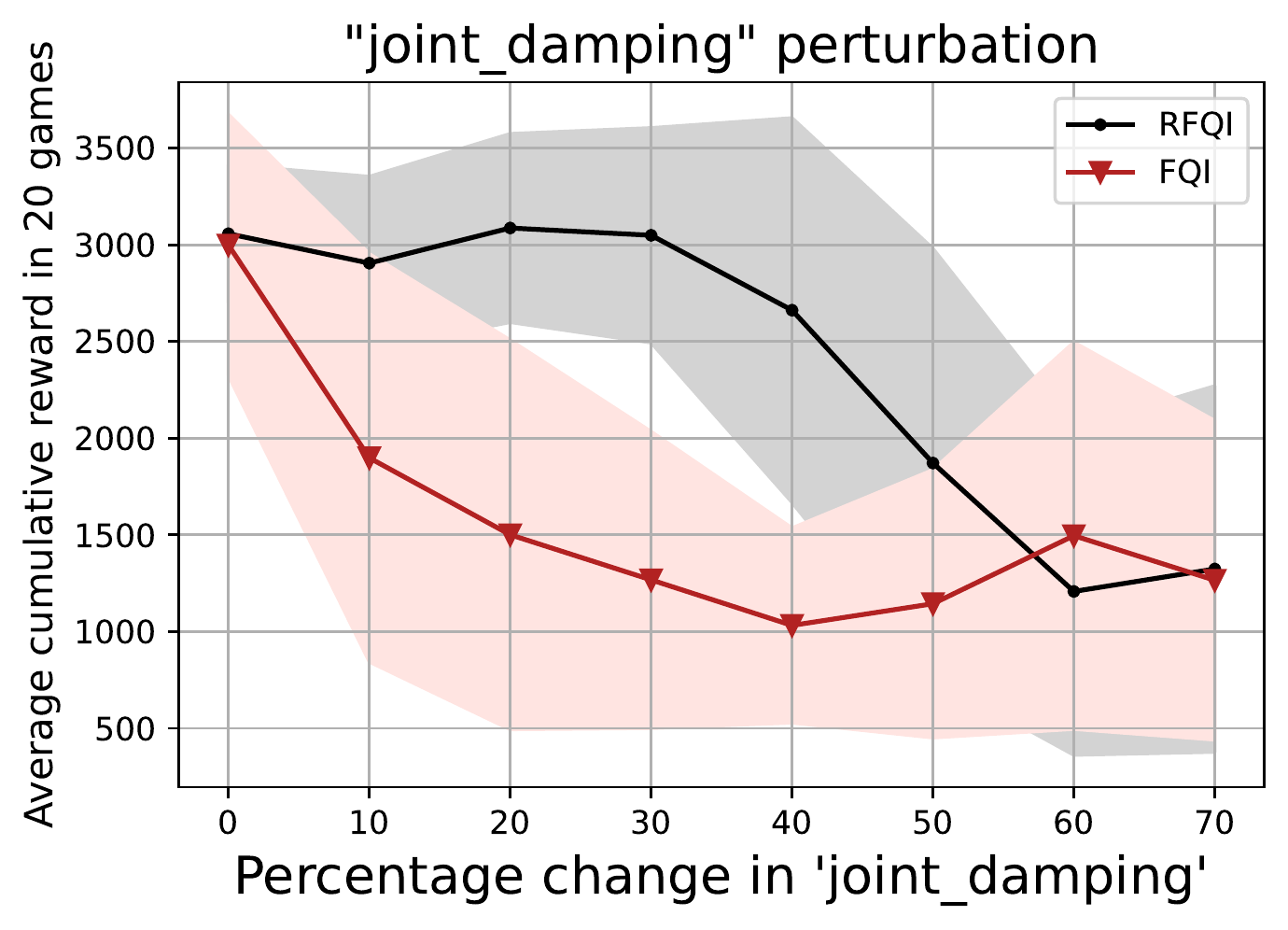}
	\end{minipage}
	\begin{minipage}{.32\textwidth}
		\centering
		\includegraphics[width=\linewidth]{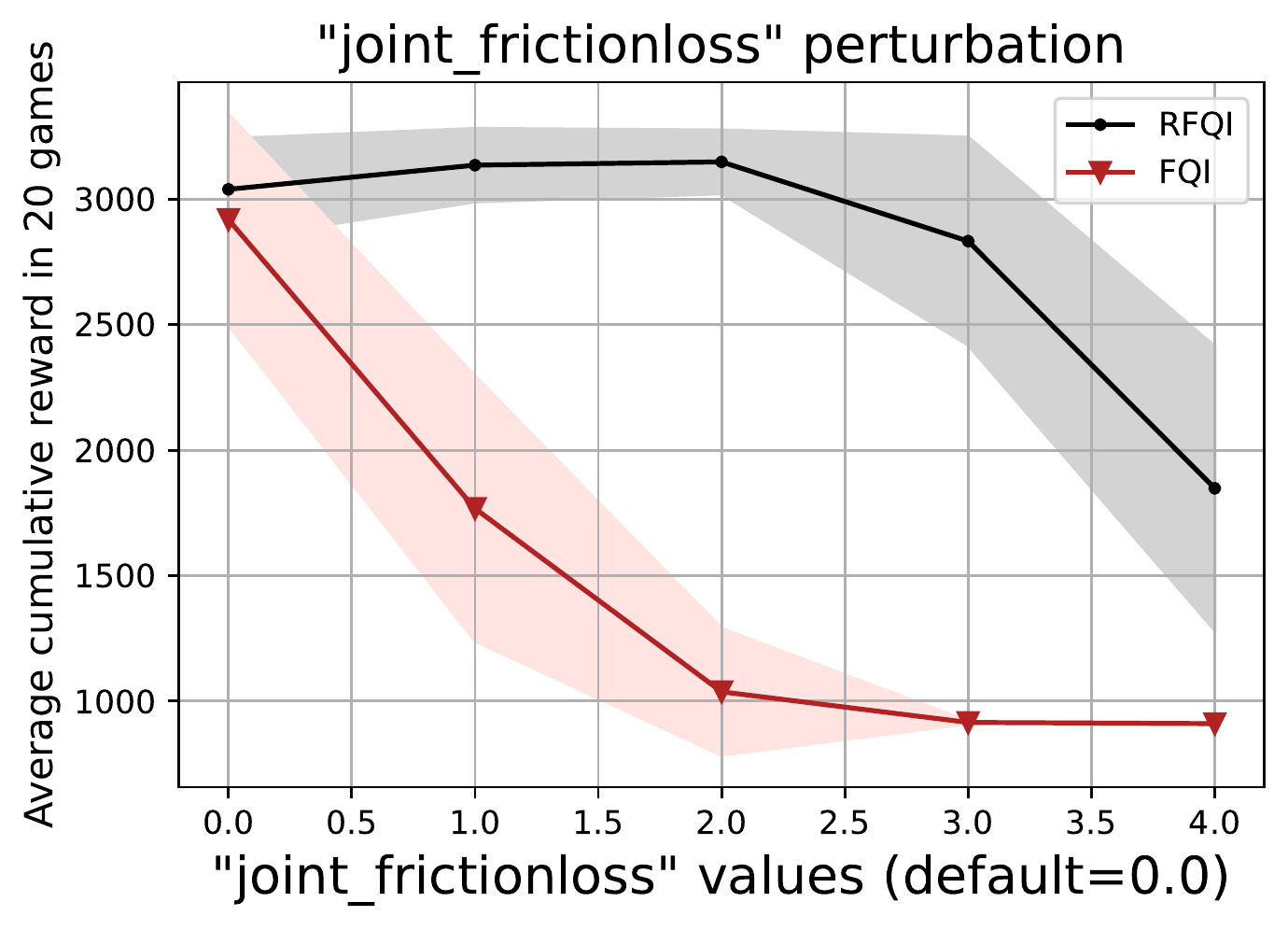}
	\end{minipage}
	\captionof{figure}{ \textit{Hopper evaluation simulation results on offline dataset $\cD_{\textrm{d}}$.} Average cumulative reward in $20$ episodes versus different model parameter perturbations mentioned in the respective titles.}
	\label{fig:hopper-d4rl}
	\centering
	\begin{minipage}{.32\textwidth}
		\centering
		\includegraphics[width=\linewidth]{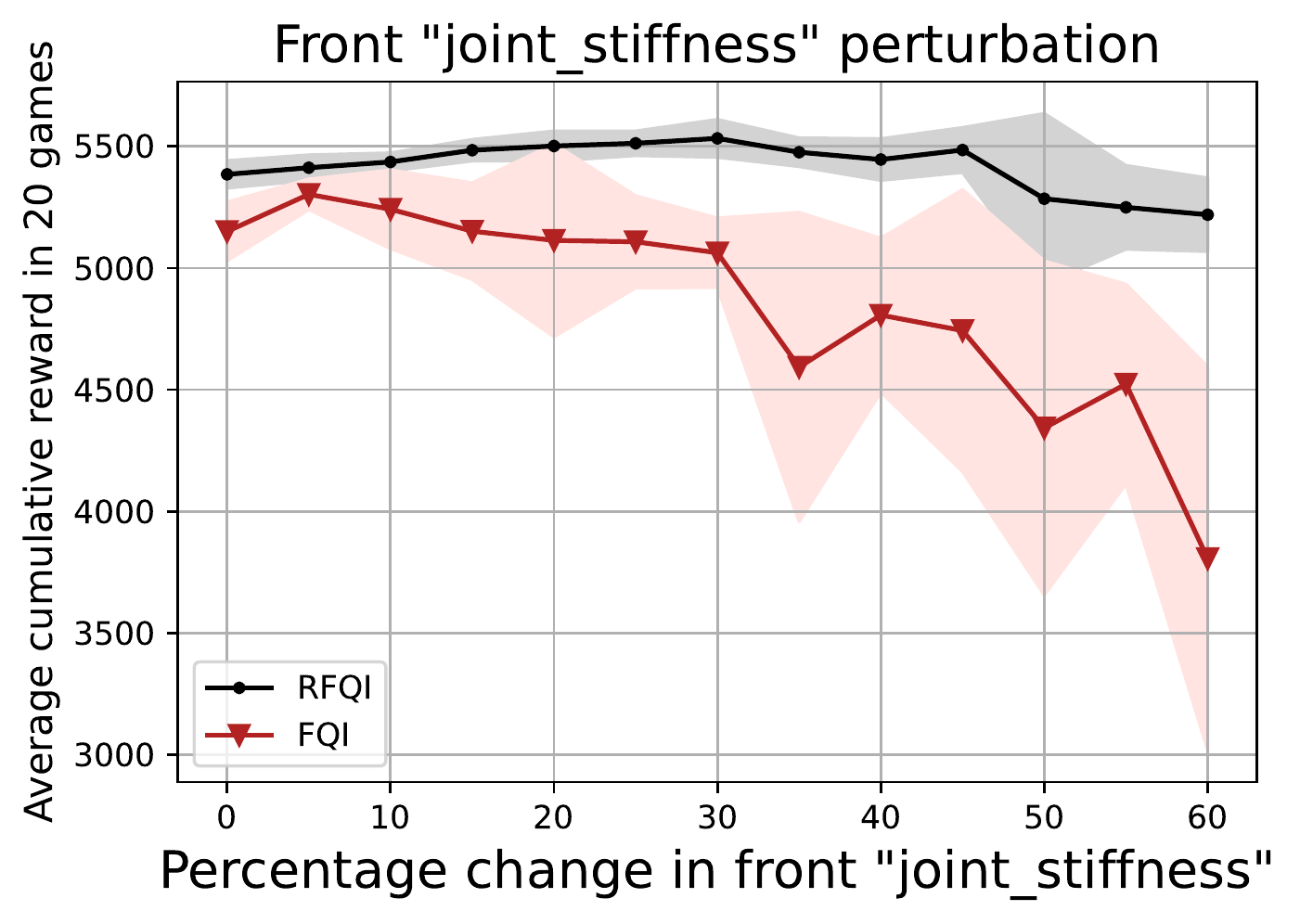}
	\end{minipage}
	\begin{minipage}{.32\textwidth}
		\centering
		\includegraphics[width=\linewidth]{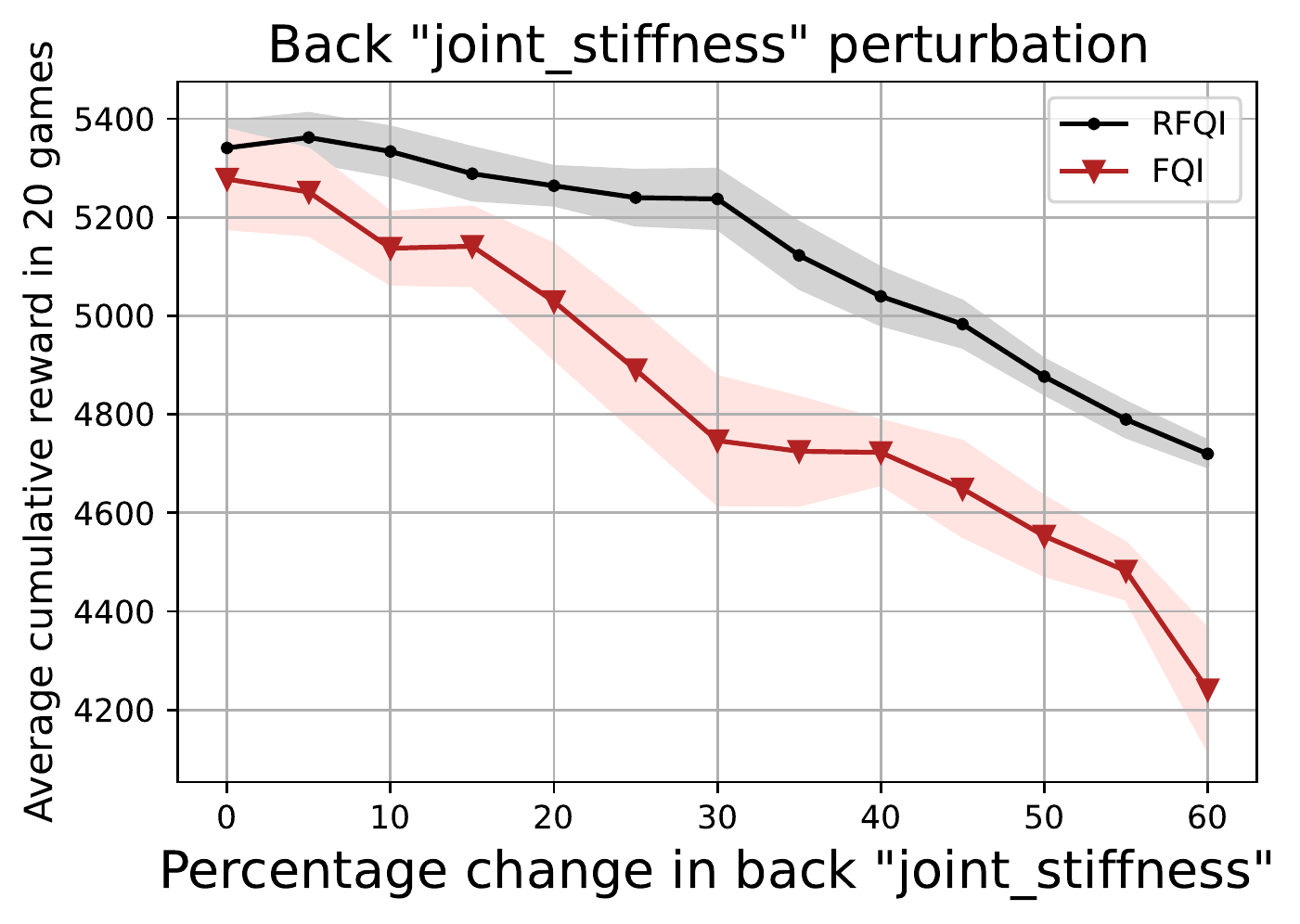}
	\end{minipage}
	\begin{minipage}{.32\textwidth}
		\centering
		\includegraphics[width=\linewidth]{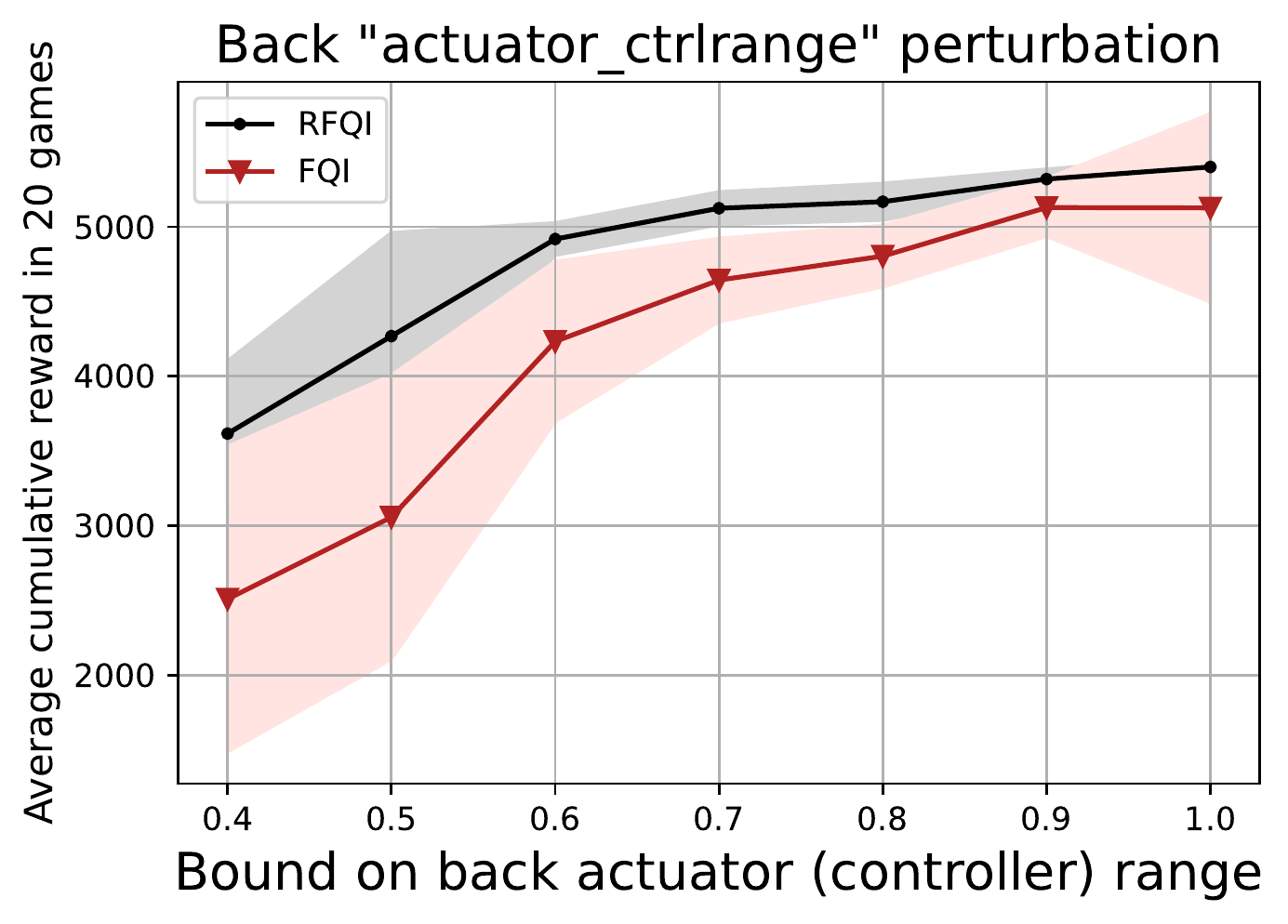}
	\end{minipage}
	\captionof{figure}{ \textit{Half-Cheetah evaluation simulation results on offline dataset $\cD_{\textrm{d}}$.} Average cumulative reward in $20$ episodes versus different model parameter perturbations mentioned in the respective titles.}
	\label{fig:half-cheetah-d4rl}
\end{figure*}

For the \textit{Hopper}, we test the robustness by changing the model parameters  \textit{gravity}, \textit{joint\_damping}, and \textit{joint\_frictionloss}. 
Fig. \ref{fig:hopper-d4rl} shows
RFQI algorithm is consistently robust compared to the non-robust FQI algorithm. We note that we found $\rho=0.5$ is the best from grid-search for RFQI algorithm.
The  average episodic reward of RFQI remains almost the same initially, and later decays much less and gracefully   when compared to FQI algorithm. For example, in plot 2 in Fig. \ref{fig:hopper-d4rl}, for the $30\%$ change in \textit{joint\_damping} parameter, the episodic reward of FQI is only around $1400$ whereas RFQI achieves an episodic reward of $3000$ which is almost the same as for unperturbed model. Similar robust performance of RFQI can be seen in other plots as well.

For the \textit{Half-Cheetah}, we test the robustness by changing the model parameters  \textit{joint\_stiffness} of front and back joints, and \textit{actuator\_ctrlrange} of back joint. 
Fig. \ref{fig:half-cheetah-d4rl} shows
RFQI algorithm is consistently robust compared to the non-robust FQI algorithm. We note that we found $\rho=0.3$ is the best from grid-search for RFQI algorithm. 
For example, in plot 1 in Fig. \ref{fig:half-cheetah-d4rl}, RFQI episodic reward stays at around $5500$ whereas FQI drops faster to $4300$ for more than $50\%$ change in the nominal value. Similar robust performance of RFQI can be seen in other plots as well.

\begin{figure*}[h]
    \begin{minipage}{.47\textwidth}
    \begin{tabular}{@{}c@{}}
     \includegraphics[width=0.5\linewidth]{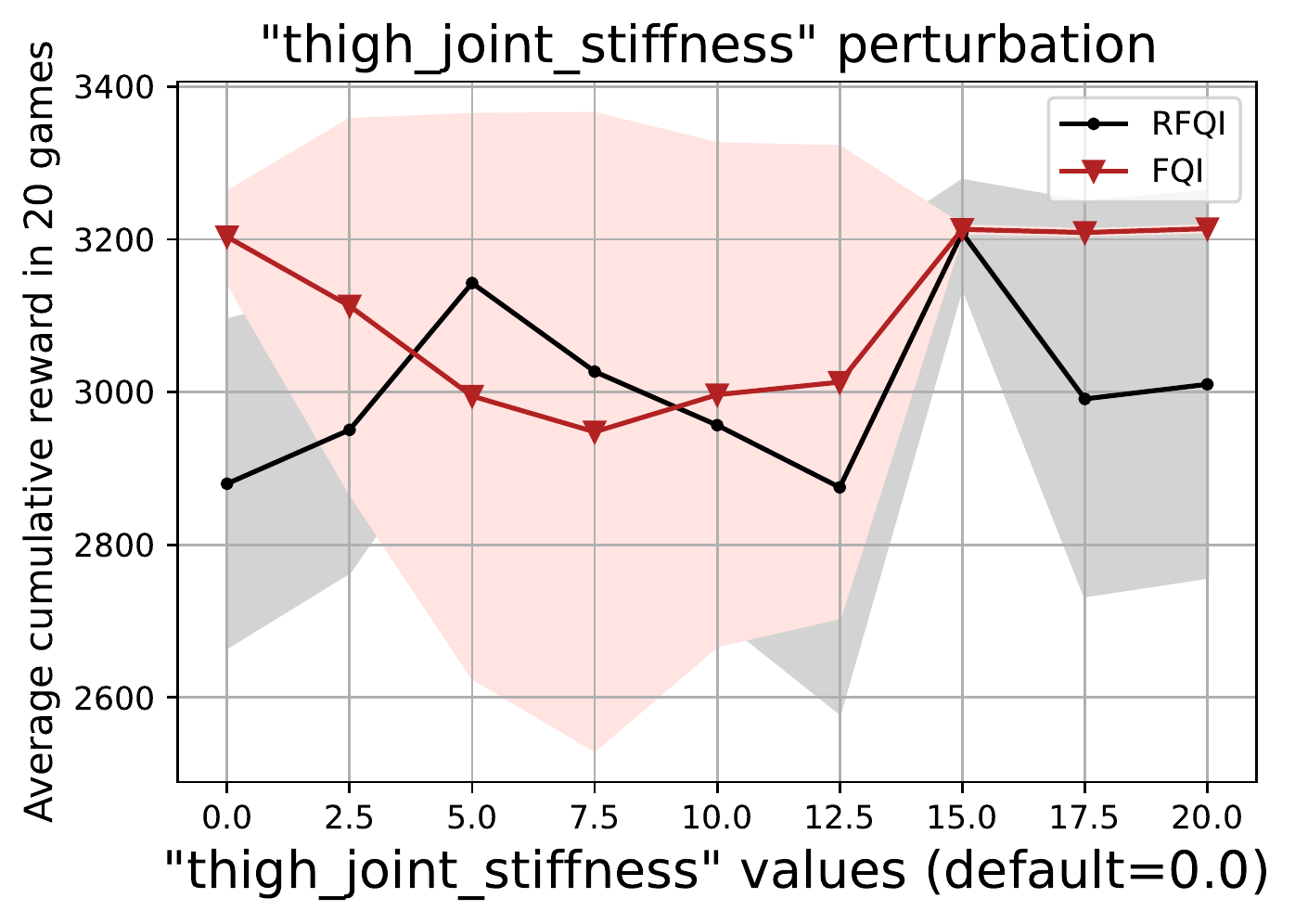}
  \includegraphics[width=0.5\linewidth]{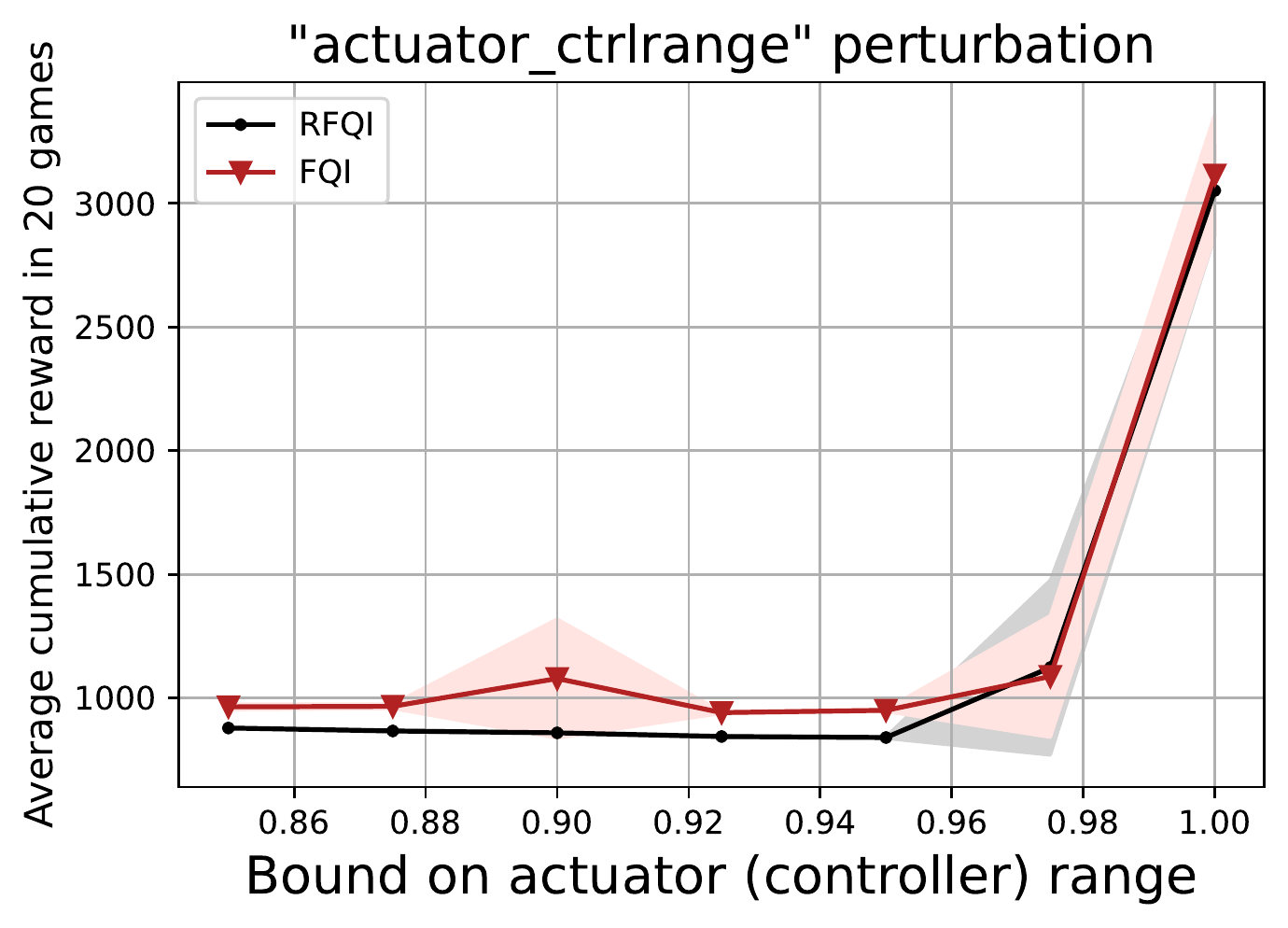}
   \end{tabular}
    \captionof{figure}{Similar performance of RFQI and FQI in  Hopper on dataset $\cD_{\textrm{d}}$ w.r.t. parameters \textit{actuator\_ctrlrange} and \textit{thigh\_joint\_stiffness}.}
	\label{fig:hopper-d4rl-same-robustness}
    \end{minipage}
    \hspace{0.5cm}
    \begin{minipage}{.47\textwidth}
    \begin{tabular}{@{}c@{}}
  \includegraphics[width=0.5\linewidth]{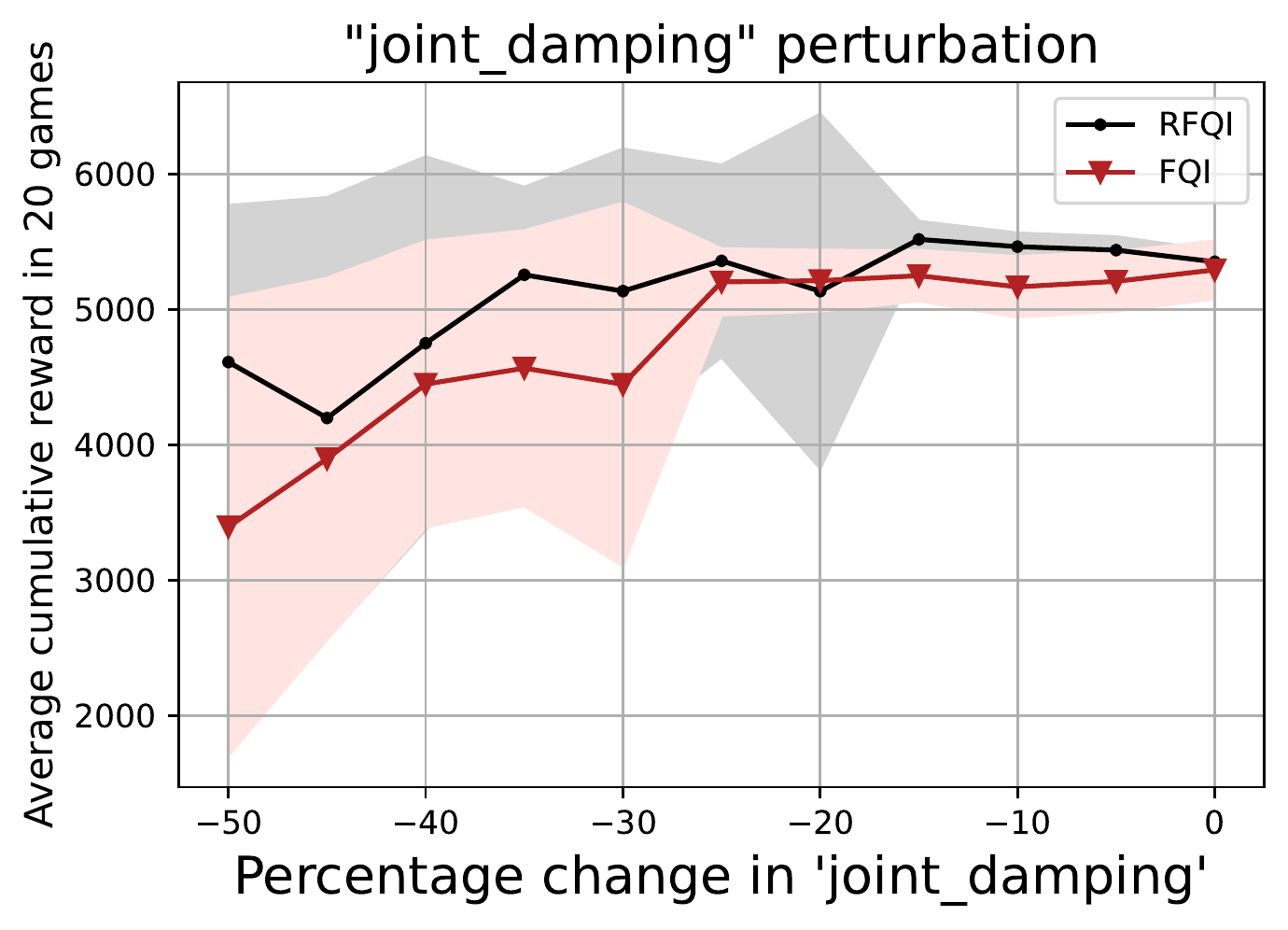}
  \includegraphics[width=0.5\linewidth]{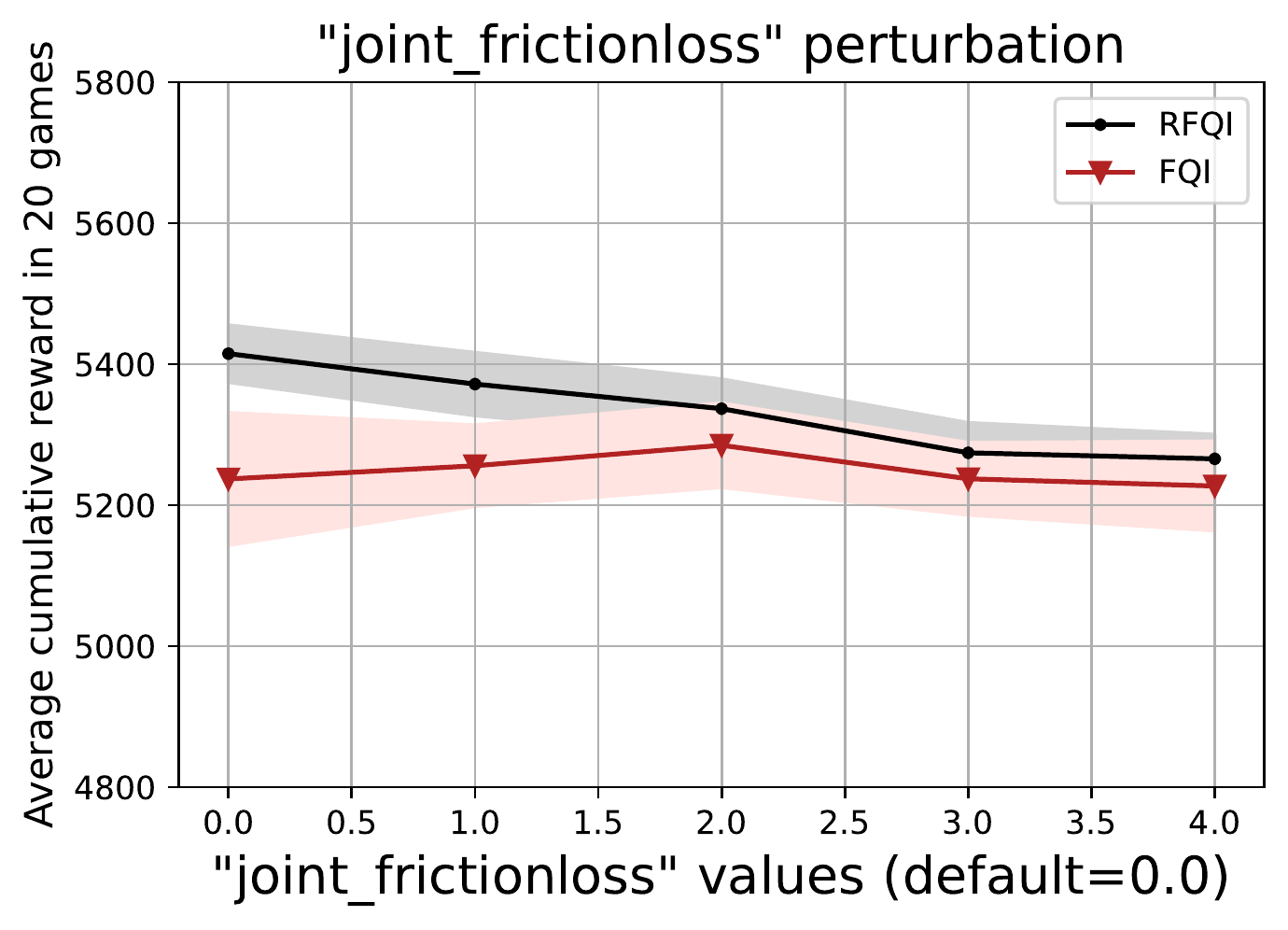}
    \end{tabular}
    \captionof{figure}{Similar performance of RFQI and FQI in  Half-Cheetah on dataset $\cD_{\textrm{d}}$ w.r.t. parameters \textit{joint\_damping} and \textit{joint\_frictionloss}.}
	\label{fig:half-cheetah-d4rl-same-robustness}
    \end{minipage}
\end{figure*}

As part of discussing the limitations of our work, we also provide two instances where RFQI and FQI algorithm behave similarly. RFQI and FQI algorithms trained on the D4RL dataset $\cD_{\textrm{d}}$ perform similarly under the perturbations of the \textit{Hopper} model parameters \textit{actuator\_ctrlrange} and \textit{thigh\_joint\_stiffness} as shown in Fig. \ref{fig:hopper-d4rl-same-robustness}. We also make similar observations under the perturbations of the \textit{Half-Cheetah} model parameters \textit{joint\_damping} (both front \textit{joint\_damping} and back \textit{joint\_damping}) and \textit{joint\_frictionloss} as shown in Fig. \ref{fig:half-cheetah-d4rl-same-robustness}. We observed that the robustness performance can depend on the offline data available, which was also observed for   non-robust offline RL algorithms \citep{liu2020provably,fu2020d4rl,levine2020offline}. Also, perturbing some parameters may make the problem really hard especially if the data is not representative with respect to that parameter. We believe that this is the  reason for the similar performance of  RFQI and FQI  w.r.t. some parameters. We believe that this opens up an exciting area of research on developing online policy gradient algorithms for robust RL, which may be able to overcome the restriction and challenges due to offline data.   We plan to pursue this goal in our future work.

\end{document}